\documentclass{article}

\usepackage[main, final]{neurips_2026}

\usepackage{amssymb}
\usepackage{algorithm}
\usepackage{algorithmic}

\usepackage[table]{xcolor}
\usepackage[utf8]{inputenc} 
\usepackage[T1]{fontenc}    
\usepackage{url}            
\usepackage{booktabs}       
\usepackage{amsfonts}       
\usepackage{nicefrac}       
\usepackage{microtype}      
\usepackage{xcolor}         
\usepackage{amsmath}
\usepackage{nicefrac}       
\usepackage{graphicx}       
\usepackage{amsmath}
\usepackage{amsthm}
\usepackage{adjustbox}
\usepackage{multirow}
\usepackage{subcaption}
\usepackage{bm}
\usepackage{minitoc}
\usepackage{colortbl}

\definecolor{mydarkred}{rgb}{0.6,0,0}
\definecolor{mydarkgreen}{rgb}{0,0.6,0}

\usepackage[
    colorlinks=true,
    linkcolor=mydarkred,
    citecolor=mydarkgreen,
    urlcolor=mydarkred
]{hyperref}

\theoremstyle{plain}
\newtheorem{theorem}{Theorem}[section]

\newtheorem{lemma}[theorem]{Lemma}

\theoremstyle{definition}

\theoremstyle{remark}
\newtheorem{remark}[theorem]{Remark}

\title{Provable Selective Auto-labeling with \\ Reliability Guarantees}
\author{%
  {\small
    \textbf{Huipeng Huang}$^{1}$\thanks{Equal contribution.}\quad
    \textbf{Wenbo Liao}$^{2,*}$\quad
    \textbf{Huajun Xi}$^{3}$\quad
    \textbf{Hao Zeng}$^{4}$\quad
    \textbf{Mengchen Zhao}$^{5}$\quad
    \textbf{Hongxin Wei}$^{1}$\thanks{Corresponding author
    (\texttt{weihx@sustech.edu.cn}).}
  }\\[0pt]
  {\small $^{1}$Department of Statistics and Data Science,
  Southern University of Science and Technology}\\[0pt]
  {\small $^{2}$Department of Mathematics,
  The Chinese University of Hong Kong}\\[0pt]
  {\small $^{3}$Department of Statistics and Data Science,
  Mohamed bin Zayed University of Artificial Intelligence}\\[0pt]
  {\small $^{4}$School of Mathematical Sciences,
  University of Electronic Science and Technology of China}\\[0pt]
  {\small $^{5}$School of Software Engineering,
  South China University of Technology}
}

\begin{document}

\maketitle

\begin{abstract}

Auto-labeling has emerged as a popular, cost-effective alternative to manual annotation.
However, its performance is often highly inconsistent across different tasks, underscoring the importance of selective strategies.
Despite this, existing heuristic methods for selective auto-labeling still rely heavily on model confidence scores and offer no reliable guarantee on the trustworthiness of the selected examples.
To address this, we propose \textbf{Conformal Labeling}, a novel method that selects a subset of auto-generated labels with a provably controlled false labeling rate (FLR).
Our key idea is to formulate selective auto-labeling as a multiple-hypothesis testing problem, where each hypothesis indicates whether to accept the auto-generated label for an instance.
Specifically, we construct a conformal $p$-value for each auto-generated label using a small calibration set, and then employ the Benjamini--Hochberg (BH) procedure with a novel correction factor to construct the subset with guaranteed FLR control.
Theoretically, we show that our method achieves finite-sample FLR control and asymptotically optimal statistical power among all $p$-value thresholding rules.
Extensive experiments on classification and open-ended generation tasks validate the effectiveness of our method, achieving high statistical power while strictly controlling the FLR.

\end{abstract}


\section{Introduction}
High-quality labeled data is essential for establishing reliable benchmarks and training high-performing models across various tasks such as image classification~\citep{5206848} and question answering~\citep{rajpurkar2016squad}.
Given the high cost of manual annotation, auto-labeling has emerged as an increasingly popular alternative in real-world applications~\citep{vishwakarma2023promises, park2024selective, zhang2025alps}.
For example, Microsoft Azure~\citep{ramanathan2024automatic} allows users to leverage GPT-3.5 Turbo~\citep{brown2020language} for labeling text datasets and GPT-4V~\citep{achiam2023gpt} for image datasets.
However, even state-of-the-art models often exhibit highly inconsistent performance across different tasks, as demonstrated by Microsoft Azure's reported accuracies ranging from $48\%$ to $94\%$~\citep{ramanathan2024automatic}.
Such variability makes it risky to blindly trust all auto-generated labels, as low-quality labels can severely degrade downstream model training and evaluation.
This highlights the importance of selective auto-labeling, which identifies and retains only high-quality subsets of auto-generated labels.

In the literature, auto-labeling methods such as CoAnnotating~\citep{Li_2023} are typically built upon simple heuristics~\citep{wang2021want}.
For example, a common practice is to filter and retain samples whose model confidence exceeds an empirically determined threshold.
Despite their cost efficiency, the labeling quality of these methods heavily depends on the reliability of model confidence, which can vary significantly across tasks. 
In other words, they provide no reliable guarantee regarding the trustworthiness of selected examples.
Consequently, this can lead to error propagation and unreliable downstream model training or evaluation, ultimately undermining the credibility of the entire pipeline. 
This motivates us to develop a selective auto-labeling method with rigorous theoretical guarantees on label reliability.

In this work, we propose \textbf{Conformal Labeling}, a novel method to select a subset of auto-generated labels with false labeling rate (FLR) control.
Our key idea is to formulate the selective auto-labeling process as a multiple-hypothesis testing problem~\citep{shaffer1995multiple, walters2016p}, where each hypothesis indicates whether to accept the auto-generated label for a sample.
Within this framework, we aim to maximize statistical power (i.e., the proportion of correctly accepted labels among all truly correct labels) while strictly controlling the FLR at a predefined level.
To this end, we construct a conformal $p$-value for each auto-generated label using a small calibration set, which provides a statistically valid measure of label reliability.
Subsequently, we select the subset by employing the BH procedure~\citep{benjamini1995controlling} with a correction factor, which is derived from the calibration set for estimating the proportion of the null hypothesis.
In this manner, our method greatly enhances the power of selective auto-labeling, offering high utility for practical applications.

Theoretically, we provide formal guarantees for Conformal Labeling in terms of both FLR control and statistical power.
We establish that Conformal Labeling achieves distribution-free finite-sample control of the FLR under the i.i.d. assumption (Theorem~\ref{thm:CL_FLR_iid}), thereby ensuring the reliability of accepted labels.
Furthermore, we prove that our method achieves asymptotically optimal statistical power among all $p$-value thresholding rules for any fixed uncertainty score function (Theorem~\ref{thm:asymptotic_power}).
In addition, we identify the optimal uncertainty score function that maximizes asymptotic power (Theorem~\ref{thm:oracle_score}), offering insights into designing more effective uncertainty scores for auto-labeling.
Overall, our method simultaneously achieves rigorous FLR control and high statistical power, demonstrating strong theoretical reliability and practical utility.



To empirically validate our method, we conduct extensive experiments on seven datasets across three labeling tasks, including image classification, LLM multi-choice QA, and open-ended generation.
Across all settings, Conformal Labeling achieves effective FLR control with high statistical power, i.e., accepting a large proportion of correct labels.
For instance, on MATH-L5~\citep{hendrycks2021math} with DeepSeek-R1-Distill-Qwen-32B~\citep{guo2025deepseek} at a target FLR of $5\%$, our method achieves an empirical FLR of $4.72\%$, whereas the compared method yields $11.32\%$. 
Through ablation studies, we show that Conformal Labeling can be applied with a small calibration set, and our selection procedure achieves significantly higher power than other selection procedures.
For example, on MATH-500~\citep{lightman2023let} with Qwen3-4B-Thinking-2507~\citep{yang2025qwen3technicalreport} at a target FLR of $5\%$, our method achieves power of $95.27\%$, while the strongest baseline yields power of $43.08\%$. 
Codes are publicly available at \href{https://github.com/ml-stat-Sustech/conformal_labeling}{https://github.com/ml-stat-Sustech/conformal\_labeling}.

We summarize our contributions as follows:
\begin{itemize}
     \item We propose \textbf{Conformal Labeling}, a novel method for selecting a subset of auto-generated labels with provable quality guarantees. 
    Regardless of the AI model's performance, Conformal Labeling guarantees the reliability of selected labels by strictly controlling the FLR.
    
   \item We prove that Conformal Labeling achieves distribution-free finite-sample FLR control.
    We further show that our method is asymptotically power-optimal among all $p$-value thresholding rules for any fixed uncertainty score.
    In addition, we identify the optimal uncertainty score function that maximizes asymptotic power under the FLR constraint.
    
    \item We empirically show that Conformal Labeling can select a large proportion of auto-generated labels while tightly controlling the FLR, through extensive experiments conducted on image classification, LLM multi-choice QA, and LLM open-ended generation tasks.
    
\end{itemize}

\section{Preliminaries}

\paragraph{Selective auto-labeling.}
We study the problem of selecting a subset of auto-generated labels with provable reliability guarantees.
Here, we give a formulation of multi-class classification as an example, while our method also applies to open-ended generation.
Let $\mathcal{X}$ denote the feature space and $\mathcal{Y}=\{1,\ldots, K\}$ denote the label space. 
The test dataset $\mathcal{D}_{\mathrm{test}} = \{X_{j}\}_{j=1}^m$ consists of $m$ instances, sampled i.i.d. from a data distribution $\mathcal{P}_{\mathcal{X}}$. 
We denote the unseen ground-truth label of instance $X_j$ as $Y_j$.
We consider a pre-trained AI model $f: \mathcal{X} \to \mathbb{R}^{|\mathcal{Y}|}$.
For an instance $X$, the auto-generated label is
$\hat{Y}=\arg\max_{y\in\mathcal{Y}} f_y(X)$, where $f_y(X)$ denotes the estimated probability of class $y \in \mathcal{Y}$.

We aim to select as many correct auto-generated labels as possible while controlling the error rate within the selected subset, leaving the remaining unlabeled instances to be handled by human annotators, stronger models, or abstention.
Formally, our goal is to identify a subset of indices $\mathcal{R}\subseteq\{1,\cdots,m\}$ from $\mathcal{D}_{\mathrm{test}}$ such that the FLR is
controlled at a user-specified level $\alpha \in (0, 1):$
\begin{equation}
\mathrm{FLR}
=
\mathbb{E}\left[
\frac{
\sum_{j=1}^{m} \mathbf{1}\{\hat{Y}_{j} \neq Y_j,\; j \in \mathcal{R}\}
}{
\max(|\mathcal{R}|,1)
}
\right]
\leq \alpha,
\end{equation}
where the quantity inside the expectation is the false labeling proportion (FLP), representing the fraction of incorrect auto-generated labels in the selected set.
The FLR quantifies the expected error rate among the selected auto-generated labels, thereby providing a direct measure of label reliability.\footnote{In Appendix~\ref{subsec:practical_flr}, we discuss the case in which the FLP for a single run exceeds the target level under FLR control.}

Since an excessively low FLR can be trivially achieved by selecting very few labels, it is also important to evaluate whether a method selects an adequate number of correct auto-generated labels. This aspect is quantified by the metric power:
\begin{equation}\label{eq:power}
    \mathrm{Power}
    =
    \mathbb{E}\left[
    \frac{
    \sum_{j=1}^{m} \mathbf{1}\{\hat{Y}_j=Y_j,j \in \mathcal{R}\}
    }{
    \max\left(\sum_{j=1}^{m}\mathbf{1}\{\hat{Y}_j =Y_j\}, 1\right)
    }
    \right].
\end{equation}
It should be emphasized that our method prioritizes FLR control over power, in that we strictly enforce FLR$\leq \alpha$ to guarantee label reliability while optimizing statistical power under the constraint.

In this work, we assume access to a small labeled calibration dataset $\mathcal{D}_{\mathrm{cal}} = \{(X_i, Y_i)\}_{i=1}^n$. 
For convenience, we denote the test dataset as $\mathcal{D}_{\mathrm{test}} = \{(X_j, Y_j)\}_{j=n+1}^{n+m}$, where $Y_j$ is not observed.
Since the labeling cost of a small dataset by human annotators is typically affordable, this assumption is practical in real-world applications.
In addition, we assume that the calibration and test samples are drawn i.i.d. from the joint distribution $\mathcal{P}_{\mathcal{X}\mathcal{Y}}$.
This assumption is realistic in auto-labeling because we sample the calibration set from a large unlabeled dataset and leave the remaining data for testing.
The process ensures that calibration and test data are naturally i.i.d.
In Appendix \ref{subsection:distributionshift}, we test our method against distribution shift to show that our proposed method is empirically robust to distribution shift.

\paragraph{Limitations of heuristic methods.}
To select high-quality auto-generated labels, existing methods rely on heuristic confidence thresholds~\citep{Li_2023, wang2024human, xiao2023freeal}.
For example, CoAnnotating~\citep{Li_2023} allocates instances to LLMs or humans according to self-consistency entropy, while other human--LLM annotation pipelines rely on verifier scores or multi-agent decisions to decide which labels should be accepted~\citep{wang2024human, xiao2023freeal}.
While such rules can reduce annotation cost, they provide no reliability guarantees on the selected samples. 
When the AI model is inaccurate or the confidence measure is poorly calibrated, the accepted subset may contain a substantial fraction of incorrect AI-generated labels. 
Such errors can propagate to downstream training or evaluation by introducing noisy supervision or biased benchmark labels, ultimately undermining the credibility of the entire pipeline.
Besides, the line of active learning can also reduce annotation cost, but it targets a different setting: selecting informative instances for human annotation to improve a model under a limited labeling budget~\citep{settles2009active, ren2021survey}.
Moreover, active learning provides no reliability guarantees on the auto-labeled subset and is not directly applicable to LLM-based labeling due to its reliance on model updates.
These limitations motivate us to design a rigorous selective auto-labeling method with provable reliability guarantees on the selected subset.


\section{Method}
\label{sec:method}

\subsection{Conformal Labeling}
Our previous section shows that existing methods cannot guarantee the reliability of selected auto-generated labels.
To address this, we propose \textbf{Conformal Labeling}, a model-agnostic method for selecting a subset of auto-generated labels with a provably controlled FLR.

Auto-labeling builds on the key insight that correct auto-generated labels tend to be associated with low uncertainty.
To quantify label reliability, our method requires an uncertainty score $\mathcal{S}:\mathcal{X}\to\mathbb{R}$, where larger values indicate less reliable labels.
For example, a natural choice of uncertainty score in classification tasks is $\mathcal{S}(X)=1-\max_{y\in\mathcal{Y}}f_y(X)$~\citep{hendrycks2016baseline}.
This score is a valid measure of uncertainty, since prior works establish that misclassified samples generally receive lower maximum probability (i.e., $\max_{y\in\mathcal{Y}}f_y(X)$) than correctly classified ones \citep{hendrycks2016baseline, tu2024does}.

\paragraph{Constructing conformal $p$-value.}
To provide a statistical guarantee, we formulate our problem as the following multiple-hypothesis testing problem:
\[
H_j^0: \hat{Y}_{n+j} \text{  is incorrect}
\quad\text{vs.}\quad
H_j^1: \hat{Y}_{n+j} \text{ is correct}.
\]


The definition of correctness is task-dependent.
In image classification, an auto-generated label is correct if it matches the ground-truth label.
In open-ended verifiable tasks, correctness corresponds to the generated output being a correct solution.
In alignment tasks, correctness means that the output aligns with human preferences.
It is worth noting that rejecting the null hypothesis $H_j^0$ indicates that $(X_{n+j},\hat{Y}_{n+j})$ should be included in the selection subset, as $\hat{Y}_{n+j}$ is deemed to be correct.

To achieve the FLR control objective, we derive $p$-values that are valid under each null hypothesis. 
Formally, for the calibration set $\mathcal{D}_{\mathrm{cal}}$, we identify the mislabeled subset $\mathcal{D}_{\mathrm{cal}}^0\subseteq\mathcal{D}_{\mathrm{cal}}$ where the auto-generated labels are incorrect.
For simplicity, we denote $\mathcal{D}_{\mathrm{cal}}^0=\{(X_i,Y_i)\}_{i=1}^{n_0}$, and thus $\hat{Y}_i$ is incorrect for $i=1,\cdots,n_0$.
We next construct the conformal $p$-values~\citep{vovk1999machine, vovk2005algorithmic} for each test label as:
\begin{equation}
\label{eq:conformal_p}
\hat{p}_j = 
\frac{
    \sum_{i=1}^{n_0} \mathbf{1}\{\mathcal{S}_i < \mathcal{S}_{n+j}\}
    +
    U_j\left(1 + \sum_{i=1}^{n_0} \mathbf{1}\{\mathcal{S}_i = \mathcal{S}_{n+j}\}\right)}{n_0+1},
\end{equation}
where \(U_j\sim\mathrm{Uniform}[0,1]\) are i.i.d. random variables to break ties.
Under the null hypothesis that the auto-generated label is incorrect, the conformal $p$-value satisfies $\mathbb{P}\{\hat{p}_j \leq \alpha \mid H_j^0 \text{ is true}\} \leq \alpha.$
Conceptually, a small \(\hat{p}_j\) means that \(\hat{Y}_{n+j}\) has unusually low uncertainty compared with known mislabeled examples, providing strong statistical evidence for correct labels.

\paragraph{Testing with an estimated correction factor.} 
To collectively test hypotheses for all test instances with controlled FLR, we can naively employ the Benjamini--Hochberg (BH) procedure \citep{benjamini1995controlling}.
However, the standard BH procedure is extremely conservative as its theoretical FLR bound scales with the proportion of true null hypotheses~\citep{storey2002direct}. 
In our auto-labeling setting, this means that BH controls the FLR below $\pi_0\alpha$~\citep{benjamini2001control,bates2023testing}, where
$\pi_0:= \mathbb{P}(H_{j}^0 \text{ is true})$ is the probability that a test auto-generated label is incorrect.
When the deployed AI model is accurate, the effective FLR level of BH collapses toward zero, causing the procedure to select only a small fraction of reliable labels and resulting in extremely low statistical power. 
This severely undermines its practical utility for selective auto-labeling, where retaining a sufficiently large set of trustworthy auto-generated labels is essential.

To mitigate this conservatism issue, we estimate $\pi_0$ from the calibration set, building upon the intuition that when the calibration data and test data are i.i.d., the model should achieve comparable error rates on both datasets.
In particular, we estimate $\pi_0$ with $\frac{1+n_0}{1+n}$, and subsequently apply the BH procedure at level $\frac{1+n}{1+n_0}\alpha$.
Concretely, let $\hat{p}_{(1)} \leq \ldots \leq \hat{p}_{(m)}$ denote the ordered statistics of the $p$-values; the selection set of Conformal Labeling is $\mathcal{R} = \{ j : \hat{p}_j \le \hat{p}_{(j^*)} \}$, where 
\begin{equation}
 j^* = \max \left\{ j : \hat{p}_{(j)} \le \frac{\alpha j (n+1)}{m(n_0+1)} \right\},     
\end{equation}
with the convention that $\max \varnothing = 0$.
We summarize the whole procedure in Algorithm~\ref{alg:CL_conformal}.



\begin{algorithm}[t!]
\caption{Conformal Labeling}
\label{alg:CL_conformal}
\begin{algorithmic}[1]
\REQUIRE Calibration set $\mathcal{D}_{\mathrm{cal}} = \{(X_i, Y_i)\}_{i=1}^n$, unlabeled test instances $\{X_{n+j}\}_{j=1}^m$, pre-trained model $f$, target FLR level $\alpha \in (0,1)$
\ENSURE A selected subset $\mathcal{R}$ with FLR control

\STATE Identify $\mathcal{D}_{\mathrm{cal}}^0 \subseteq \mathcal{D}_{\mathrm{cal}}$ with size $n_0$

\STATE \textbf{Step 1: Compute uncertainty scores}
\FOR{$i = 1$ to $n+m$}
    \STATE Compute uncertainty score $\mathcal{S}_i$
\ENDFOR

\STATE \textbf{Step 2: Construct conformal $p$-values}
\FOR{$j = 1$ to $m$}
    \STATE Construct conformal $p$-value $\hat{p}_j$ according to Eq.~\eqref{eq:conformal_p}
\ENDFOR

\STATE \textbf{Step 3: Testing with an estimated correction factor}
\STATE Compute $j^* = \max \left\{ j : \hat{p}_{(j)} \le \frac{\alpha j(n+1)}{m(n_0+1)} \right\}$, where $\hat{p}_{(j)}$ is the $j$-th smallest $p$-value.

\STATE \textbf{Return} $\mathcal{R} = \{ j : \hat{p}_j \le \hat{p}_{(j^*)} \}$
\end{algorithmic}
\end{algorithm}

\subsection{Theoretical Analysis}

We first establish the finite-sample guarantee of our method.
The proof is provided in Appendix~\ref{subsection:proof of CL}.

\begin{theorem}[Finite-sample FLR control]\label{thm:CL_FLR_iid}
Suppose the calibration samples $(X_i, Y_i)_{i=1}^n$ and test samples $(X_{n+j}, Y_{n+j})_{j=1}^m$ are i.i.d.
For any target level $\alpha\in(0,1)$, let $\mathcal{R}$ be the selection set returned by Algorithm~\ref{alg:CL_conformal}.
Define $\pi_0 = \mathbb{P}(H_{j}^0\text{ is true})$, the probability that an auto-generated label is incorrect.
Taking expectation over the randomness of the calibration and test data, the FLR satisfies:
\[
\mathrm{FLR}(\mathcal R) \leq \left[1 - (1 - \pi_0)^{n+1}\right] \alpha \leq \alpha.
\]
\end{theorem}




\begin{remark}
Theorem~\ref{thm:CL_FLR_iid} establishes that applying BH with the correction factor \(\frac{1+n}{1+n_0}\) preserves finite-sample FLR control.
The correction factor enables our method to adaptively adjust the selection threshold according to the estimated accuracy of the AI model.
As a result, our method provides tighter FLR control than the standard BH procedure and achieves higher statistical power.
\end{remark}

Beyond finite-sample validity, we further characterize the asymptotic FLR and power of our method and show its power optimality under mild assumptions.
The proof is provided in Appendix~\ref{subsection:asymptotic_optimality}.

\begin{theorem}[Asymptotic power optimality]
\label{thm:asymptotic_power}
Assume that \(\mathcal S(X)\) has continuous conditional distributions under
\(H_j^0\) and \(H_j^1\). Let
$
    G_0(t):=\mathbb P(\hat p_j\le t\mid H_j^0),
    \ \
    G_1(t):=\mathbb P(\hat p_j\le t\mid H_j^1).
$
As \(n,m\to\infty\), Conformal Labeling is asymptotically power-optimal among all
\(p\)-value thresholding procedures. That is, for any
thresholding procedure \(\mathcal R_t=\{j:\hat p_j\le t\}\) with FLR at most $\alpha$,
we have
\[
    \mathrm{Power}_{\infty}(\mathcal R_t)
    \le
    \mathrm{Power}^{\mathrm{CL}}_\infty.
\]

More specifically, the Conformal Labeling threshold converges to
\[
    t_{\mathrm{CL}}
    :=
    \sup\left\{
        t\in[0,1]:
        t\le \frac{\alpha}{\pi_0}
        \left[\pi_0G_0(t)+(1-\pi_0)G_1(t)\right]
    \right\}.
\]
Its asymptotic FLR and power are given by
\[
    \mathrm{FLR}^{\mathrm{CL}}_\infty
    =
    \frac{\pi_0 G_0(t_{\mathrm{CL}})}
    {\pi_0G_0(t_{\mathrm{CL}})+(1-\pi_0)G_1(t_{\mathrm{CL}})}
    ,
    \qquad
    \mathrm{Power}^{\mathrm{CL}}_\infty
    =
    G_1(t_{\mathrm{CL}}).
\]
\end{theorem}

\begin{remark}
Theorem~\ref{thm:asymptotic_power} shows that Conformal Labeling achieves optimal power asymptotically among all thresholding rules based on the same conformal \(p\)-values, demonstrating the superiority of our method.
The superior power is also validated by our empirical results in Table~\ref{tab:compare_selection_results}.
\end{remark}

\subsection{Choice of uncertainty score}
Theorem~\ref{thm:asymptotic_power} shows that, once the conformal \(p\)-values are fixed, Conformal Labeling is asymptotically power-optimal under mild assumptions.
Therefore, further power improvement mainly depends on the choice of the uncertainty score for conformal \(p\)-value construction.
We now characterize the oracle choice of uncertainty score in Theorem~\ref{thm:oracle_score}.
The proof is provided in Appendix~\ref{subsection:optimal_score}.

\begin{theorem}[Asymptotically optimal uncertainty score]
\label{thm:oracle_score}
As \(n,m\to\infty\), the power-maximizing uncertainty score function for Conformal Labeling is
\[
    S^\star(X_{n+j})
    :=
    \mathbb P(H_j^0 \text{ is true}\mid X_{n+j}).
\]
\end{theorem}

\begin{remark}
Theorem~\ref{thm:oracle_score} identifies the oracle uncertainty score as the posterior probability that the auto-generated label is incorrect.
For classification, this score reduces to
$
    S^\star(X)
    =
    1-\max_y \mathbb{P}_y(X),
$
where $\mathbb{P}_y(X)$ denotes the underlying posterior probability, not an AI model's predicted probability.
\end{remark}


However, the oracle score is typically unavailable due to the miscalibration of model confidence in practice~\citep{guo2017calibration}.
To provide practical guidance for score function designs, Figure~\ref{fig:score_comparisons} compares three uncertainty scores by visualizing their resulting $p$-value distributions and comparing their power when employed with our method.
We evaluate three uncertainty scores on ImageNet using a pre-trained ResNet-34: one minus maximum softmax probability (MSP)~\citep{hendrycks2016baseline}, energy score~\citep{liu2020energy}, and DOCTOR-\(\alpha\) score (\(D_{\alpha}\))~\citep{granese2021doctor}. See Appendix~\ref{section:score} for an overview of these score functions.
The results show that one minus MSP and \(D_{\alpha}\) provide a clear distinction between correct and incorrect predictions, thus achieving high power.
However, the energy score fails to provide this separation, leading to low power.
Since one minus MSP achieves high power and is a practical proxy for the oracle score function, we use it as the default uncertainty score in the main experiments.

\begin{figure*}[t!]
    \centering
    \begin{subfigure}[b]{0.24\textwidth}
        \centering
        \includegraphics[width=0.95\textwidth]{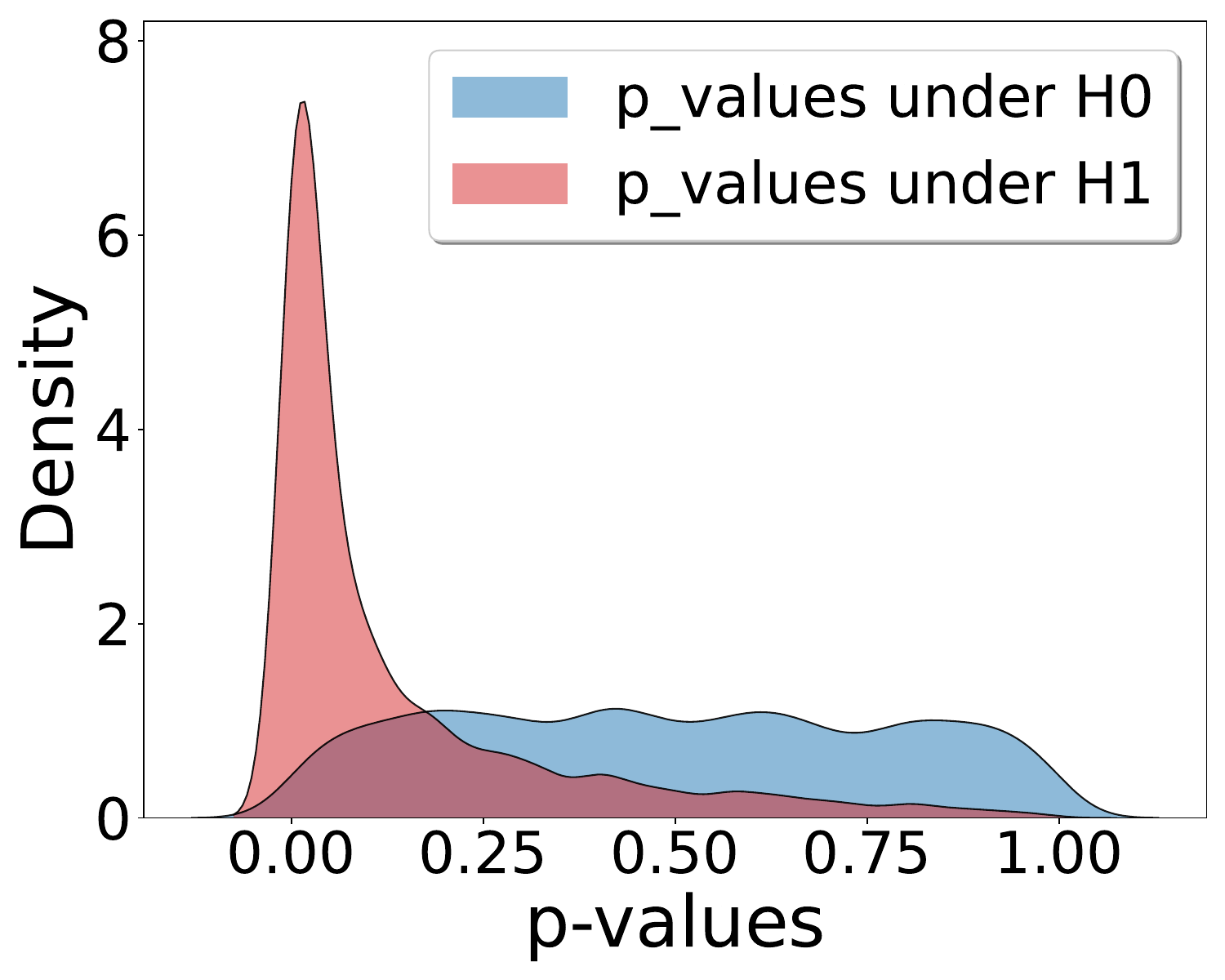}
        \caption{1 - MSP}
        \label{fig:sub1}
    \end{subfigure}
    \hfill
    \begin{subfigure}[b]{0.24\textwidth}
        \centering
        \includegraphics[width=0.95\textwidth]{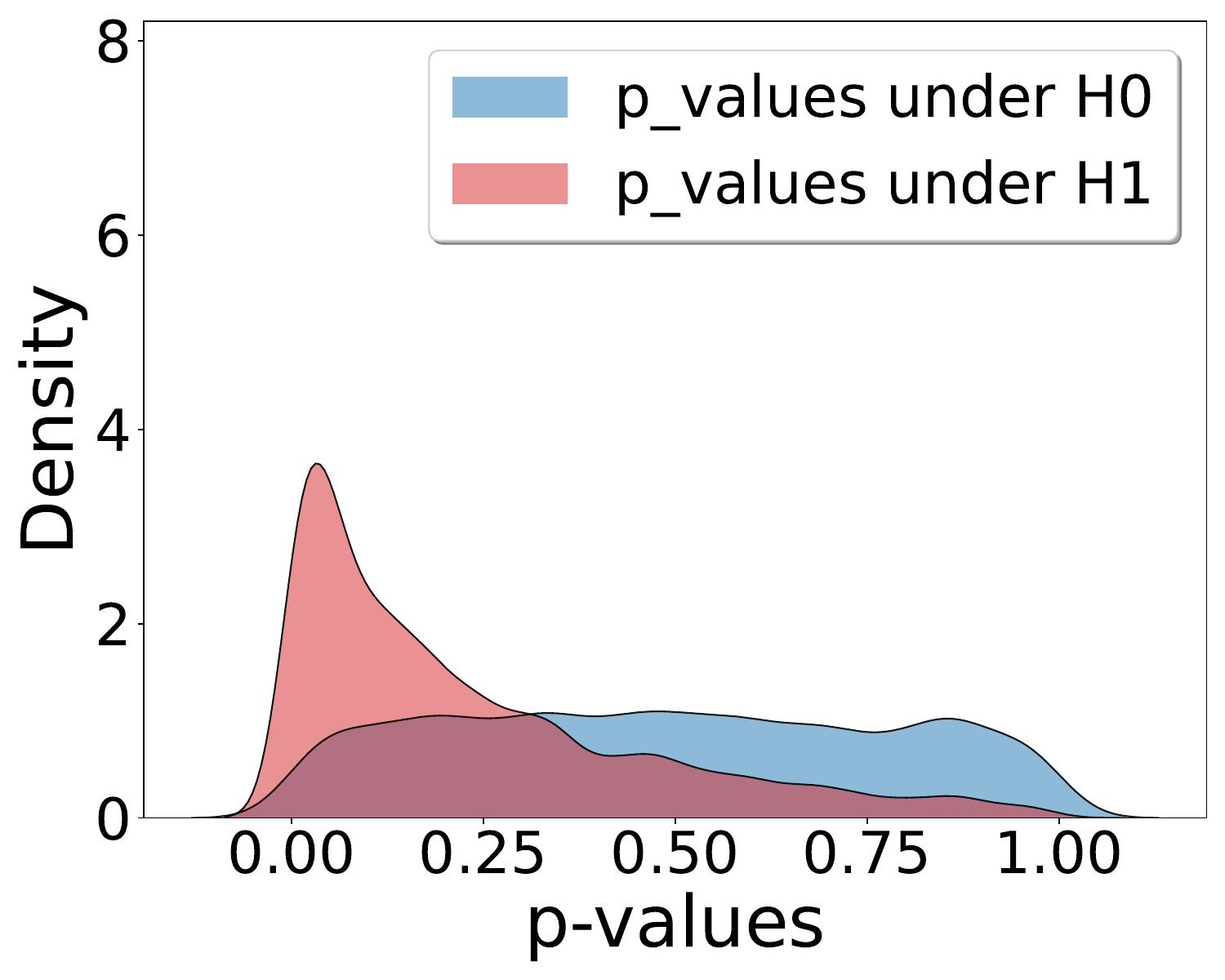}
        \caption{Energy}
        \label{fig:sub2}
    \end{subfigure}
    \hfill
    \begin{subfigure}[b]{0.24\textwidth}
        \centering
        \includegraphics[width=0.95\textwidth]{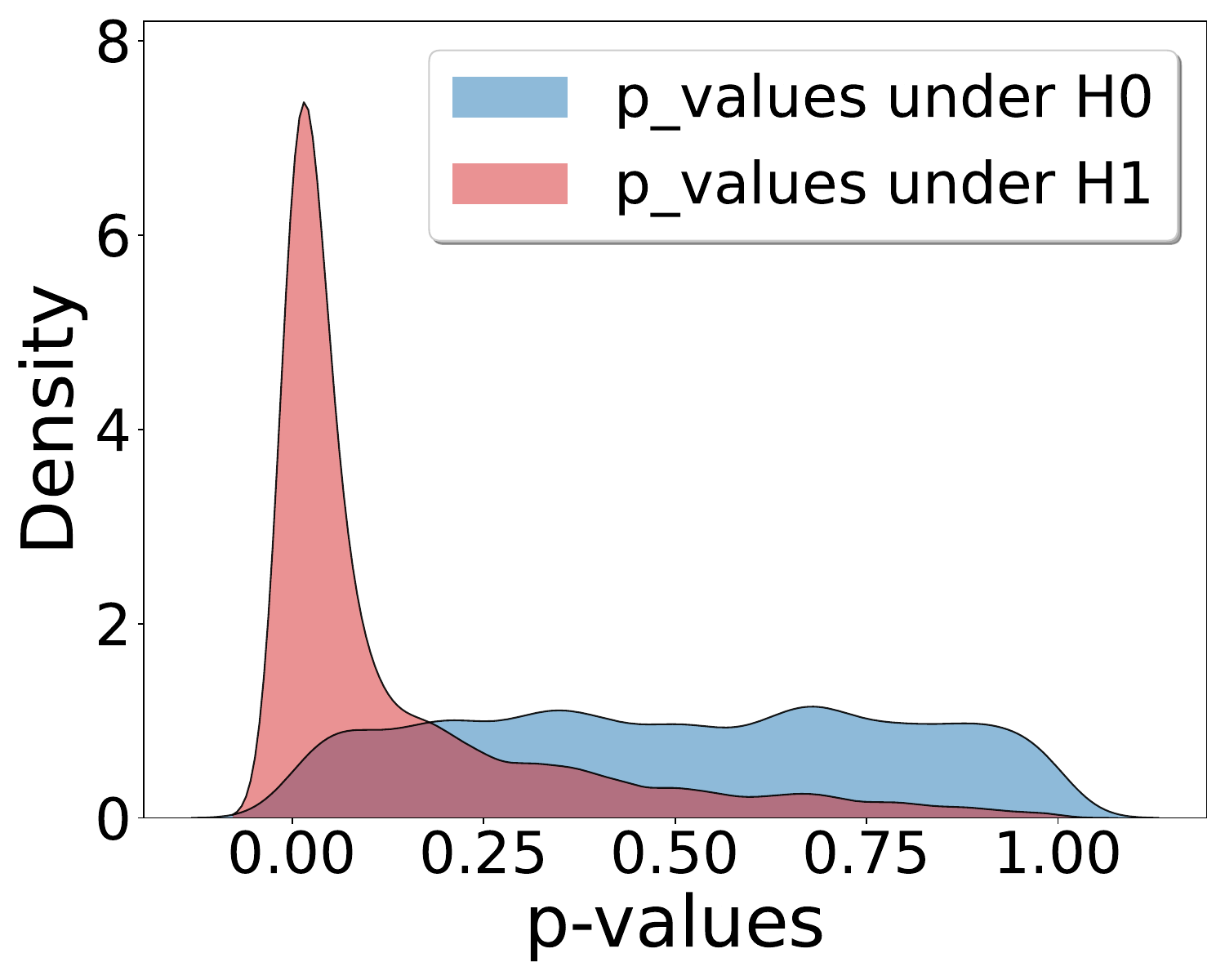}
        \caption{$D_{\alpha}$}
        \label{fig:sub3}
    \end{subfigure}
    \hfill
    \begin{subfigure}[b]{0.24\textwidth}
        \centering
        \includegraphics[width=0.95\textwidth]{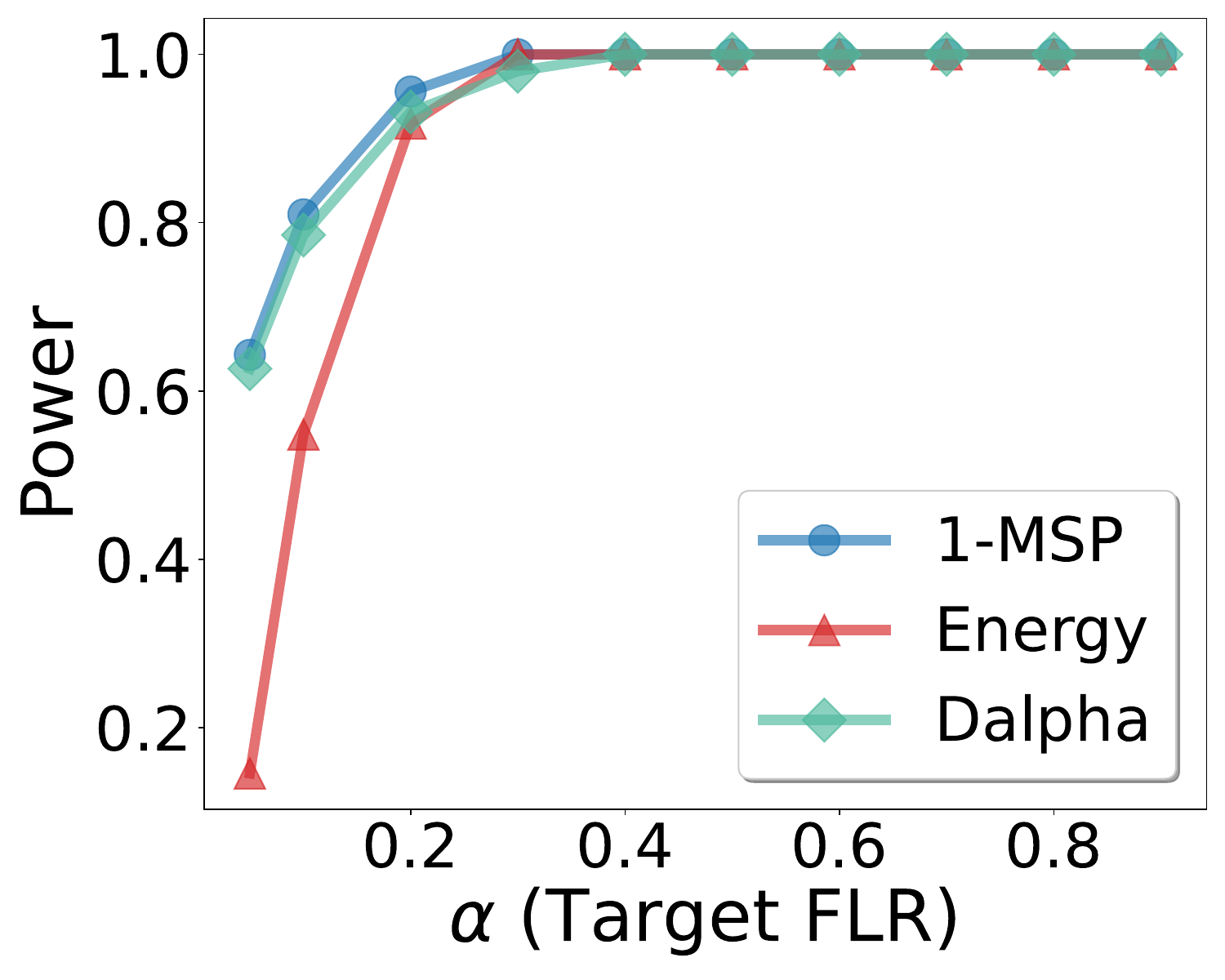}
        \caption{Power comparison}
        \label{fig:sub4}
    \end{subfigure}
    \caption{\textbf{Empirical distributions and power of conformal $p$-values under different uncertainty scores.}
We compare three uncertainty scores on ImageNet with a ResNet-34 classifier: one minus maximum softmax probability (MSP), energy, and DOCTOR-$\alpha$ ($D_{\alpha}$) scores.
}
    \label{fig:score_comparisons}
\end{figure*}

\section{Experiments}
\subsection{Setup}
\label{sec:setup}

\paragraph{Datasets.}
We conduct experiments on image classification, LLM multi-choice QA, and open-ended generation tasks.
For image classification, we use ImageNet \citep{5206848} and ImageNet-V2 \citep{recht2019imagenetclassifiersgeneralizeimagenet}.
For multi-choice QA, we use MedMCQA \citep{pal2022medmcqalargescalemultisubject} and MMLU~\citep{hendrycks2021measuring}.
For the open-ended generation task, we use MATH-500~\citep{lightman2023let}, MATH-L5~\citep{hendrycks2021math}, and $\text{HumanEval}^+$~\citep{liu2023your} to evaluate Conformal Labeling's performance in math reasoning and coding tasks.

\paragraph{Models.}
We conduct extensive experiments on various open-source AI models.
For image labeling,  we use ResNet-34 \citep{he2015deepresiduallearningimage}, DenseNet-161 \citep{huang2018denselyconnectedconvolutionalnetworks}, and a Vision-Language Model CLIP \citep{radford2021learningtransferablevisualmodels}, which is based on a pre-trained ViT-B/32 \citep{dosovitskiy2021imageworth16x16words}.
For LLM QA tasks, we employ Llama-3.1-8B-Instruct \citep{dubey2024llama}, Qwen3-32B \citep{yang2025qwen3technicalreport}, and Llama-3.1-70B-Instruct.
For LLM open-ended generation tasks, we use Qwen3-4B-Instruct-2507, Qwen3-4B-Thinking-2507, and DeepSeek-R1-Distill-Qwen-32B \citep{guo2025deepseek}.
The sampling temperature and other hyperparameters are detailed in Appendix \ref{section:moreimplementation}.

\paragraph{Uncertainty scores.}
For image classification and LLM QA tasks, we use one minus MSP as the uncertainty score.
For open-ended generation tasks, we adopt the logits-based bottom-10\% group uncertainty score~\citep{fu2025deep} in the main experiments.
We also provide results under the self-consistency-based uncertainty score for math reasoning tasks in Appendix~\ref{subsec:self-consistency-confidence}, where it achieves higher power than the logits-based score.
We provide details for these score functions in Appendix \ref{sec:oe_score}.


\paragraph{Baselines and evaluation metrics.}
We compare our selection procedure with BH-style procedures, including standard BH~\citep{benjamini1995controlling}, Storey-BH~\citep{storey2002direct}, and Quantile-BH~\citep{benjamini2006adaptive} procedures, which target the same FLR objective.
We also compare our method with selection with guaranteed risk control (SGR) \citep{geifman2017selective}.
Since SGR controls mFLR rather than FLR, we view it as a contextual baseline with a related but different reliability guarantee.
Additional comparisons with probably approximately correct (PAC) labeling \citep{candes2024probablyapproximately} and two heuristic baselines are provided in Appendix~\ref{subsec:pac_labeling} and Appendix~\ref{subsec:heuristic_baseline}.
We provide an introduction to all these baselines in Appendix \ref{sec:baselines}.
We evaluate the selection performance using the following metrics: (1) FLR, (2) power, and (3) selection ratio, defined as the number of selected test labels divided by the test set size.

\paragraph{Implementation details.}
Each experiment is repeated 1000 times, and we report the mean FLR, power, and selection ratio.
For image classification and LLM multi-choice QA tasks, we randomly select 1,000 instances as the calibration set.
For open-ended generation, we use dataset-specific calibration set sizes, as reported in Table~\ref{tab:datasets_split}.
More details of implementation are provided in Appendix~\ref{section:moreimplementation}.

\subsection{Results}

\begin{table*}[t!]
\centering
\caption{\textbf{Labeling performance of Conformal Labeling on image labeling and QA tasks.}
We set $\alpha=0.1$ for Conformal Labeling, and $\alpha=0.1$, $\epsilon=0.2$ for SGR.
\textbf{Bold} indicates the best result.}

\label{tab:results_image_qa}
\setlength{\tabcolsep}{2pt}
\footnotesize
\renewcommand{\arraystretch}{1}

\begin{tabular}{c c c c *{3}{c} *{3}{c}}
\toprule
\multirow{2}{*}{Task} & \multirow{2}{*}{Dataset} & \multirow{2}{*}{Model}
& \multicolumn{1}{c}{AI only}
& \multicolumn{3}{c}{SGR}
& \multicolumn{3}{c}{Conformal Labeling} \\
\cmidrule(lr){4-4} \cmidrule(lr){5-7} \cmidrule(lr){8-10}
 & & 
 & Error (\%)
 & FLR (\%) & Power (\%) & Ratio (\%)
 & FLR (\%) & Power (\%) & Ratio (\%) \\
\midrule

\multirow{6}{*}{Image}
& \multirow{3}{*}{ImageNet}
& ResNet-34
& 26.71
& 7.45 & 73.17 & 57.98
& 9.82 & 79.44 & \textbf{64.60} \\
& & DenseNet-161
& 22.89
& 7.66 & 78.16 & 66.13
& 9.98 & 84.91 & \textbf{72.76} \\
& & CLIP-ViT-B/32
& 40.35
& 4.87 & 23.52 & 15.02
& 9.82 & 45.13 & \textbf{29.94} \\
\cmidrule(lr){2-10}
& \multirow{3}{*}{ImageNet-V2}
& ResNet-34
& 39.04
& 6.63 & 52.03 & 34.05
& 10.00 & 61.97 & \textbf{42.03} \\
& & DenseNet-161
& 34.88
& 6.85 & 53.97 & 37.87
& 9.96 & 67.03 & \textbf{48.54} \\
& & CLIP-ViT-B/32
& 47.78
& 2.18 & 6.88 & 3.85
& 9.86 & 35.00 & \textbf{20.40} \\
\midrule

\multirow{6}{*}{QA}
& \multirow{3}{*}{MedMCQA}
& Llama-3.1-8B-Instruct
& 40.35
& 3.72 & 13.37 & 8.50
& 9.94 & 31.79 & \textbf{21.14} \\
& & Qwen3-32B
& 33.44
& 4.64 & 21.93 & 15.68
& 9.93 & 51.70 & \textbf{38.31} \\
& & Llama-3.1-70B-Instruct
& 28.90
& 7.21 & 61.97 & 47.56
& 9.91 & 69.96 & \textbf{55.19} \\
\cmidrule(lr){2-10}
& \multirow{3}{*}{MMLU}
& Llama-3.1-8B-Instruct
& 35.91
& 6.55 & 45.49 & 31.30
& 9.89 & 57.11 & \textbf{40.68} \\
& & Qwen3-32B
& 21.43
& 7.61 & 76.05 & 64.73
& 10.00 & 82.92 & \textbf{72.43} \\
& & Llama-3.1-70B-Instruct
& 18.24
& 7.67 & 82.50 & 73.12
& 9.96 & 88.18 & \textbf{80.11} \\
\bottomrule
\end{tabular}
\end{table*}


\paragraph{Conformal Labeling achieves strict FLR control with high power.}
Table \ref{tab:results_image_qa} reports the labeling performance of Conformal Labeling on image classification and LLM QA tasks.
A salient observation is that Conformal Labeling consistently controls the FLR below the target level while selecting a substantial proportion of correct auto-generated labels.
For example, on ImageNet with DenseNet-161, our method achieves an FLR of $9.98\%$ and selects $72.76\%$ of test instances with a power of $84.91\%$.
The tightest FLR control and highest power are also consistently observed in other datasets and models.
Overall, these results demonstrate that Conformal Labeling can reliably accept a large subset of auto-generated labels while keeping the FLR below the target level.

\begin{table*}[t!]
\centering
\caption{\textbf{Labeling performance of Conformal Labeling on open-ended generation tasks.}
We set $\alpha=0.05$ for Conformal Labeling, and $\alpha=0.05$, $\epsilon=0.2$ for SGR.
\textbf{Bold} indicates the best result.}

\label{tab:results_oe}
\setlength{\tabcolsep}{2pt}
\footnotesize
\renewcommand{\arraystretch}{1}

\begin{tabular}{c c c *{3}{c} *{3}{c}}
\toprule
\multirow{2}{*}{Dataset} & \multirow{2}{*}{Model} 
& \multicolumn{1}{c}{AI only}
& \multicolumn{3}{c}{SGR}
& \multicolumn{3}{c}{Conformal Labeling} \\
\cmidrule(lr){3-3} \cmidrule(lr){4-6} \cmidrule(lr){7-9}
 & 
 & Error (\%)
 & FLR (\%) & Power (\%) & Ratio (\%) 
 & FLR (\%) & Power (\%) & Ratio (\%) \\
\midrule

\multirow{3}{*}{MATH-500} 
& Qwen3-4B-Instruct-2507      
& 6.60
& 1.83 & 45.10 & 43.72
& 4.91 & 94.75 & \textbf{93.10} \\
& Qwen3-4B-Thinking-2507    
& 5.20
& 1.69 & 44.28 & 42.87
& 4.73 & 95.27 & \textbf{94.86} \\
& DeepSeek-R1-Distill-Qwen-32B  
& 6.20
& 0.04 & 0.59 & 0.57
& 4.83 & 47.82 & \textbf{47.56} \\
\cmidrule(lr){1-9}

\multirow{3}{*}{MATH-L5} 
& Qwen3-4B-Instruct-2507      
& 8.88
& 0.58 & 6.95 & 6.60
& 4.97 & 78.39 & \textbf{75.35} \\
& Qwen3-4B-Thinking-2507    
& 4.30
& 2.17 & 52.31 & 51.45
& 4.12 & 97.88 & \textbf{97.76} \\
& DeepSeek-R1-Distill-Qwen-32B  
& 5.96
& 11.32 & 1.10 & 1.16
& 4.72 & 51.43 & \textbf{51.28} \\
\cmidrule(lr){1-9}

\multirow{3}{*}{$\text{HumanEval}^+$} 
& Qwen3-4B-Instruct-2507      
& 7.32
& 0.00 & 1.07 & 0.99
& 4.78 & 75.46 & \textbf{73.88} \\
& Qwen3-4B-Thinking-2507    
& 2.32
& 0.59 & 0.99 & 1.00
& 2.43 & 100.00 & \textbf{100.00} \\
& DeepSeek-R1-Distill-Qwen-32B  
& 0.61
& 0.00 & 0.99 & 0.98
& 0.67 & 100.00 & \textbf{100.00} \\

\bottomrule
\end{tabular}
\end{table*}

\paragraph{Conformal Labeling is effective for open-ended generation tasks.}
Table~\ref{tab:results_oe} reports the labeling performance of Conformal Labeling on open-ended generation tasks.
The results show that Conformal Labeling reliably controls the FLR below the target level while selecting substantially more correct auto-generated labels than the baseline.
For example, on MATH-L5 with DeepSeek-R1-Distill-Qwen-32B, our method controls the FLR at $4.72\%$ and selects $51.28\%$ of test instances, whereas SGR incurs an FLR of $11.32\%$ and selects only $1.16\%$ of instances.
When the underlying model is highly accurate, Conformal Labeling can accept nearly all generated labels while maintaining low FLR.
For instance, on HumanEval+ with Qwen3-4B-Thinking-2507, our method achieves $100\%$ selection ratio, whereas SGR selects only $1\%$ of instances.
This suggests that SGR can be conservative under its theoretical guarantee, leading to an unfavorable accuracy--power trade-off, whereas Conformal Labeling achieves high power with tight FLR control in the open-ended generation tasks.

\paragraph{Conformal Labeling outperforms existing selection procedures.}
To isolate the effect of selection procedures, we fix the $p$-values and employ different selection procedures.
The comparison results in Table~\ref{tab:compare_selection_results} demonstrate the superiority of our selection procedure.
Conformal Labeling consistently achieves the highest power among all the procedures while controlling the FLR.
For example, on MATH-500 with Qwen3-4B-Instruct-2507 at $\alpha=0.05$, our method achieves 94.75\% power, substantially outperforming BH, Storey-BH, and Quantile-BH, which obtain power of 0.10\%, 58.66\%, and 21.11\%, respectively.
This corresponds to an absolute improvement of more than 30 percentage points over the strongest baseline.
In summary, our method achieves the best labeling performance, attaining the highest power and tightest FLR control among all compared selection procedures.

\begin{table*}[t!]
\centering
\caption{\textbf{Comparison of selection procedures across three labeling tasks.}
We compare the selection procedure of Conformal Labeling with three baselines:
BH, Storey-BH, and Quantile-BH procedures.
We report results at $\alpha=0.1$ for image classification and LLM multi-choice QA, and at $\alpha=0.05$ for open-ended generation.
All values are percentages.
\textbf{Bold} indicates the best result.
}
\label{tab:compare_selection_results}

\footnotesize
\setlength{\tabcolsep}{0pt}
\renewcommand{\arraystretch}{1}

\makebox[\textwidth][c]{%
\begin{tabular}{
@{\hskip4pt}l
@{\hskip3pt}l
@{\hskip3pt}c
@{\hskip13pt}c
@{\hskip15pt}c
@{\hskip13pt}c
@{\hskip15pt}c
@{\hskip13pt}c
@{\hskip15pt}c
@{\hskip13pt}c
@{\hskip4pt}}
\toprule
\multirow{2}{*}{Dataset} & \multirow{2}{*}{Model} 
& \multicolumn{2}{c}{BH} 
& \multicolumn{2}{c}{Storey-BH} 
& \multicolumn{2}{c}{Quantile-BH} 
& \multicolumn{2}{c}{CL} \\
\cmidrule(lr){3-4} \cmidrule(lr){5-6} \cmidrule(lr){7-8} \cmidrule(lr){9-10}
 & & FLR & Power & FLR  & Power & FLR & Power & FLR & Power \\
\midrule

\rowcolor{blue!5}
\multicolumn{10}{c}{\textbf{Image Classification}}\\[2pt]
\multirow{3}{*}{ImageNet} 
& ResNet-34        & 2.60 & 45.39 & 8.51 & 75.57 & 7.93 & 73.44 & 9.82 & \textbf{79.44} \\
& DenseNet-161     & 2.34 & 44.04 & 8.17 & 80.71 & 7.34 & 78.07 & 9.98 & \textbf{84.91} \\
& CLIP-ViT-B/32    & 3.84 & 19.36 & 8.55 & 40.42 & 8.12 & 39.11 & 9.82 & \textbf{45.13} \\
\cmidrule(lr){1-10}
\multirow{3}{*}{ImageNet-V2} 
& ResNet-34        & 3.45 & 35.14 & 9.67 & 61.06 & 8.80 & 58.82 & 10.00 & \textbf{61.97} \\
& DenseNet-161     & 3.37 & 28.28 & 9.22 & 64.42 & 8.35 & 61.28 & 9.96 & \textbf{67.03} \\
& CLIP-ViT-B/32    & 4.16 & 12.51 & 8.86 & 30.75 & 8.34 & 28.75 & 9.86 & \textbf{35.00} \\
\midrule

\rowcolor{green!4}
\multicolumn{10}{c}{\textbf{Multi-choice QA}}\\[2pt]
\multirow{3}{*}{MedMCQA} 
& Llama-3.1-8B-Instruct   & 3.96 & 15.14 & 7.91 & 26.84 & 6.81 & 24.51 & 9.94 & \textbf{31.79} \\
& Qwen3-32B               & 2.63 & 10.48 & 6.92 & 34.40 & 6.33 & 31.39 & 9.93 & \textbf{51.70} \\
& Llama-3.1-70B-Instruct  & 2.72 & 30.72 & 7.96 & 63.74 & 6.98 & 60.59 & 9.91 & \textbf{69.96} \\
\cmidrule(lr){1-10}
\multirow{3}{*}{MMLU} 
& Llama-3.1-8B-Instruct   & 3.57 & 29.80 & 7.80 & 50.20 & 7.17 & 48.18 & 9.89 & \textbf{57.11} \\
& Qwen3-32B               & 1.61 & 12.94 & 7.78 & 75.91 & 6.42 & 70.87 & 10.00 & \textbf{82.92} \\
& Llama-3.1-70B-Instruct  & 1.40 & 15.93 & 7.52 & 81.45 & 5.89 & 76.43 & 9.96 & \textbf{88.18} \\
\midrule

\rowcolor{yellow!5}
\multicolumn{10}{c}{\textbf{Open-ended Generation}}\\[2pt]
\multirow{3}{*}{MATH-500} 
& Qwen3-4B-Instruct-2507       & 0.33 & 0.10 & 3.50 & 58.66 & 1.83 & 21.11 & 4.91 & \textbf{94.75} \\
& Qwen3-4B-Thinking-2507       & 0.00 & 0.38 & 1.67 & 43.08 & 0.79 & 13.53 & 4.73 & \textbf{95.27} \\
& DeepSeek-R1-Distill-Qwen-32B & 0.00 & 0.09 & 0.00 & 0.05 & 1.00 & 0.09 & 4.83 & \textbf{47.82} \\
\cmidrule(lr){1-10}
\multirow{3}{*}{MATH-L5} 
& Qwen3-4B-Instruct-2507       & 0.00 & 0.08 & 3.91 & 52.46 & 1.61 & 10.64 & 4.97 & \textbf{78.39} \\
& Qwen3-4B-Thinking-2507       & 0.00 & 0.06 & 2.15 & 56.30 & 1.25 & 3.40 & 4.12 & \textbf{97.88} \\
& DeepSeek-R1-Distill-Qwen-32B & 0.00 & 0.04 & 0.00 & 0.02 & 0.00 & 0.04 & 4.72 & \textbf{51.43} \\
\bottomrule
\end{tabular}%
}
\end{table*}

\begin{figure}[t]
    \centering
    \includegraphics[width=1.0\textwidth]{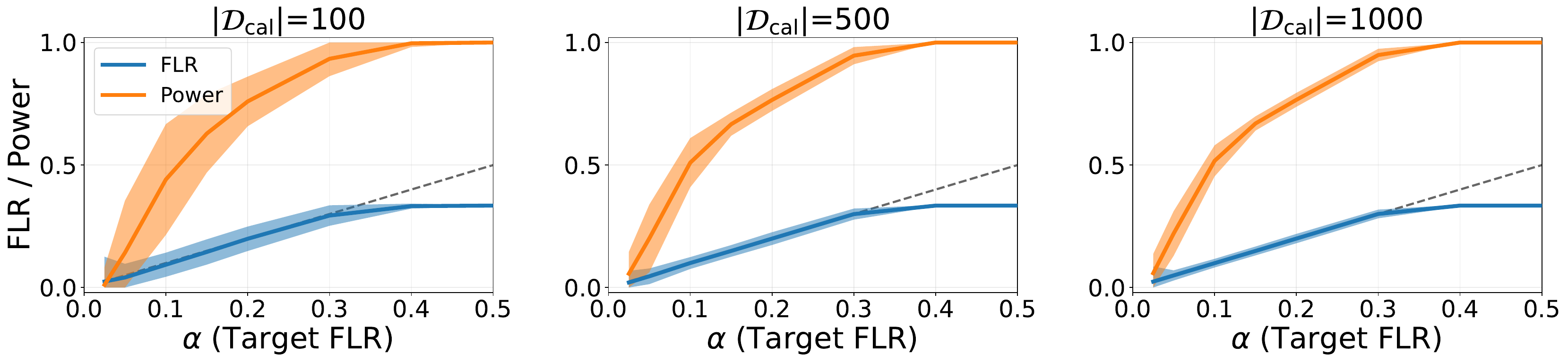}
    \caption{\textbf{Labeling performance of Conformal Labeling with varying calibration set sizes on MedMCQA with Qwen3-32B.} 
    The shaded areas represent standard deviations.}
    \label{fig:calibration_variance}

\end{figure}

\paragraph{How many calibration samples are needed?}
In Figure~\ref{fig:calibration_variance}, we study how the calibration set size affects the performance of Conformal Labeling on MedMCQA.
Specifically, we vary $|\mathcal{D}_{\mathrm{cal}}|$ in $\{100,500,1000\}$ and report the realized FLR and power with their standard deviations.
In Appendix~\ref{subsec:n_0}, we also ablate how $n_0$ affects the performance of our method.
The results show that increasing the calibration set size reduces the variance of both FLR and power, leading to more stable selection results.
This effect is particularly clear when moving from $|\mathcal{D}_{\mathrm{cal}}|=100$ to $|\mathcal{D}_{\mathrm{cal}}|=500$, where the shaded regions become substantially narrower.
Further increasing the calibration size to 1000 brings additional variance reduction, but the improvement is marginal, indicating clear performance saturation.
Overall, Conformal Labeling can be applied with a small calibration set.


\paragraph{Additional results.}
Due to space constraints, we defer additional analyses to the appendix.
Appendix~\ref{subsec:group-wise-FLR} studies group-wise FLR and presents an extension of Conformal Labeling that controls the group-wise FLR separately within each group.
Appendix~\ref{subsec:edge_case} discusses the edge case where no incorrect labels are observed in the calibration set ($n_0=0$).
Appendix~\ref{subsec:end_to_end} investigates whether the selected labels can be used to improve the training of a neural network on a synthetic dataset.

\section{Related work}
\label{sec:related-work}

\paragraph{Conformal Inference.}
Conformal inference \citep{vovk2005algorithmic, gazin2025selecting, humbert2025online} is a statistical framework characterized by distribution-free and model-agnostic theoretical guarantees.
In the literature, two popular directions have emerged: conformal novelty detection \citep{bates2023testing, bashari2023derandomized, wu2024conditional, bashari2025robust, lee2025full} and conformal selection \citep{jin2023selection, gui2024conformal}.
The former aims to identify out-of-distribution instances, with recent advances focusing on enhancing selection power by incorporating various forms of side information \citep{10.1093/jrsssb/qkad138, marandon2024adaptive, zhao2025conformalizedempiricalbayesmethod}.
The latter aims to select candidates whose unobserved outcomes exceed user-specified values \citep{jin2023selection}, which has been extended to multivariate data selection \citep{bai2025multivariate, hao2026multicondition} and online data selection \citep{xu2024online, liu2025online}.
In particular, conformal alignment~\citep{gui2024conformal} extends conformal selection to identify foundation model outputs that meet a user-specified alignment criterion.
However, it requires an additional training set to fit an alignment score predictor (see Appendix~\ref{subsec:conformal_alignment} for a detailed discussion).
In addition, some works have developed conformal inference techniques for classification tasks \citep{zhao2023controlling, sun2025unified}. 
For example, recent work \citep{zhao2023controlling} proposes controlling the overall error rate in binary and multi-class classification, while a concurrent work \citep{sun2025unified} extends this to control general group-wise error rates within a unified framework.
In this work, we motivate and tailor the framework of conformal selection for selective labeling and support its application to open-ended generation tasks.

\paragraph{Selective auto-labeling.}
Selective auto-labeling is an increasingly important topic for constructing cost-efficient high-quality datasets \citep{gu2012selective, vrabac2022medselect}.
In the literature, several works instantiate selective auto-labeling through human--LLM collaborative annotation pipelines, where model uncertainty, verifier scores, small-model predictions, or multi-agent decisions are used to decide label acceptance~\citep{Li_2023, wang2024human, xiao2023freeal, qin2025crowdagent}.
However, these methods determine label acceptance based on empirical heuristics without providing reliability guarantees on the selected auto-generated labels. 
Another line of work studies threshold-based auto-labeling~\citep{vishwakarma2023promises}, which iteratively trains a classifier and auto-labels instances whose confidence exceeds a threshold estimated on a calibration set.
Colander~\citep{vishwakarma2024pearls} further learns a confidence function that better separates correct and incorrect predictions to maximize the auto-labeling coverage.
These methods rely on model retraining, which is often infeasible for LLM-based labeling, and their selection procedure provides no theoretical guarantee (see Appendix~\ref{subsec:colander} for a detailed discussion).
A concurrent work~\citep{candes2024probablyapproximately} proposes PAC labeling, which guarantees small dataset-level errors by combining AI labels with error-free expert labels.
This differs from our goal of controlling the FLR within the accepted auto-generated labels (see Appendix~\ref{subsec:pac_labeling} for a detailed discussion).
Our method is also related to selective prediction, where a model is allowed to abstain from making predictions when uncertain \citep{JMLR:v11:el-yaniv10a, mozannar2020consistent, yang2023uncertainty, kamath2020selective, yoshikawa2023selective}.
Several methods in this line of work provide guarantees on marginal FLR~\citep{JMLR:v11:el-yaniv10a, geifman2017selective}, which is a different error measure from FLR and is less preferred than FLR by prior researchers~\citep{benjamini1995controlling, storey2003positive}.
We provide a discussion of these two error measures in Appendix~\ref{subsec:mFLR_vs_FLR}.

\section{Conclusion}\label{section:limitation}
In this paper, we introduce Conformal Labeling, a statistically principled framework for selective auto-labeling with provable reliability guarantees.
Our method formulates selective auto-labeling as a multiple-hypothesis testing problem.
Within this framework, we construct a conformal \(p\)-value for each auto-generated label using a small labeled calibration set.
We further derive a novel calibration-based correction factor to estimate the null proportion, and apply the BH procedure with this correction factor to select reliable auto-generated labels with high statistical power.
We prove that Conformal Labeling achieves finite-sample FLR control and is asymptotically power-optimal among \(p\)-value thresholding rules under the same uncertainty score.
Extensive experiments on image classification, multi-choice QA, and open-ended generation tasks show that our method selects a large fraction of correct auto-generated labels while controlling FLR.
We hope this work inspires future research on trustworthy selective auto-labeling.

\paragraph{Limitations.}
Conformal Labeling relies on the i.i.d. assumption between calibration and test samples for its FLR guarantees. 
Although this assumption is natural in our transductive setting and our experiments in Appendix~\ref{subsection:distributionshift} suggest empirical robustness under moderate distribution shift, the theoretical guarantees do not hold when the assumption is violated. 
We believe it will be interesting to develop conformal labeling algorithms that are robust to covariate shift or label shift in the future.

\begin{ack}
This research is supported by Guangdong Basic and Applied Basic Research Foundation (Grant No. 2026A1515011367).
This project is also supported by the Jiangsu Provincial Key Discipline Construction Project (Statistics) and the open project of Joint Lab for Statistics and Finance (Grant No.~2025JLSF101).
Hao Zeng is supported by the SUSTech-NUS Joint Research Program.
We gratefully acknowledge Shenzhen Guangming Science City Development and Construction Co., Ltd.\ and its Material Genome Big-Science Facilities Platform for providing technical support and assistance in data collection and analysis.
The authors declare no competing interests.
\end{ack}

\bibliography{main}

\newpage
\appendix
\onecolumn

\section{Proofs}\label{section:proof}



\subsection{Proof of Theorem \ref{thm:CL_FLR_iid}}\label{subsection:proof of CL}

\begin{lemma}(\citet{benjamini2001control}, Theorem~1.2)\label{lem:prdsFLR}
If the joint distribution of the $p$-values is PRDS on the subset corresponding to true null hypotheses, applying the Benjamini--Hochberg procedure at level $\alpha$ guarantees
\[
\text{FLR} \leq \frac{m_0}{m} \alpha,\]
where $m_0$ is the number of true null hypotheses.
\end{lemma}

\begin{lemma}\label{lem:FLPbound}
Under the assumptions of Theorem~\ref{thm:CL_FLR_iid}, where the calibration samples \((X_i, Y_i)_{i=1}^n\) and the test samples \((X_{n+j}, Y_{n+j})_{j=1}^m\) are independently and identically distributed (i.i.d.), and conditional on \(n_0\) (the number of true null hypotheses in the calibration set) and \(m_0\) (the number of true null hypotheses in the test set), the expected false labeling proportion (FLP) satisfies
\[
\mathbb{E}[\text{FLP} \mid n_0, m_0] \leq \frac{1 + n}{1 + n_0} \frac{m_0}{m} \alpha,
\]
where \(\alpha \in (0, 1)\) is the target FLR level, and the selection set \(\mathcal{R}\) is determined by Algorithm~\ref{alg:CL_conformal}.
\end{lemma}

\begin{proof}
Fix the realized values of \(n_0\) and \(m_0\), and let \(r:=n_0+m_0\).
Let
\[
\mathcal I_0:=\{i\in[n]:\hat Y_i\neq Y_i\}=\{i_1<\cdots<i_{n_0}\}
\]
be the set of mislabeled calibration indices, and recall that
\[
\mathcal H_0=\{j\in[m]:\hat Y_{n+j}\neq Y_{n+j}\}=\{j_1<\cdots<j_{m_0}\}.
\]
Define the vector of null scores by
\[
Z:=(Z_1,\dots,Z_r)
:=
(\mathcal S_{i_1},\dots,\mathcal S_{i_{n_0}},\mathcal S_{n+j_1},\dots,\mathcal S_{n+j_{m_0}}).
\]

We first show that \(Z\) is exchangeable conditional on \((n_0,m_0)\). 
It suffices to prove that for any permutation \(\pi\) of \(\{1,\dots,r\}\),
\[
f_{Z\mid n_0,m_0}(z_1,\dots,z_r)
=
f_{Z\mid n_0,m_0}(z_{\pi(1)},\dots,z_{\pi(r)})
\]
for all \((z_1,\dots,z_r)\in\mathbb R^r\).

To this end, for any subsets \(I=\{u_1<\cdots<u_{n_0}\}\subseteq[n]\) and \(J=\{v_1<\cdots<v_{m_0}\}\subseteq[m]\), define
\[
E_{I,J}:=\{\mathcal I_0=I,\ \mathcal H_0=J\}.
\]
On the event \(E_{I,J}\), we have
\[
Z=(\mathcal S_{u_1},\dots,\mathcal S_{u_{n_0}},\mathcal S_{n+v_1},\dots,\mathcal S_{n+v_{m_0}}).
\]
Therefore,
\[
f_{Z\mid n_0,m_0}(z_1,\dots,z_r)
=
\sum_{\substack{I\subseteq[n],\,|I|=n_0\\ J\subseteq[m],\,|J|=m_0}}
f_{Z\mid E_{I,J}}(z_1,\dots,z_r)\,
\mathbb P(E_{I,J}\mid n_0,m_0).
\]

Now fix \(I=\{u_1<\cdots<u_{n_0}\}\) and \(J=\{v_1<\cdots<v_{m_0}\}\). Under the event \(E_{I,J}\), the vector \(Z\) consists exactly of the scores of those samples whose predictions are incorrect. Since the samples
\[
(X_1,Y_1),\dots,(X_n,Y_n),(X_{n+1},Y_{n+1}),\dots,(X_{n+m},Y_{n+m})
\]
are i.i.d., and both \(\mathcal S_i\) and the event \(\{\hat Y_i\neq Y_i\}\) are determined by the \(i\)-th sample alone, conditioning on \(E_{I,J}\) acts coordinatewise. Hence the coordinates of \(Z\) remain conditionally independent given \(E_{I,J}\), and each coordinate has the same conditional density \(f_{\mathcal S\mid \hat Y\neq Y}\). Therefore,
\[
f_{Z\mid E_{I,J}}(z_1,\dots,z_r)
=
\prod_{k=1}^r f_{\mathcal S\mid \hat Y\neq Y}(z_k).
\]
In particular, this density is symmetric in \((z_1,\dots,z_r)\), so for every permutation \(\pi\),
\[
f_{Z\mid E_{I,J}}(z_1,\dots,z_r)
=
f_{Z\mid E_{I,J}}(z_{\pi(1)},\dots,z_{\pi(r)}).
\]
Substituting this into the above mixture representation gives
\[
\begin{aligned}
f_{Z\mid n_0,m_0}(z_1,\dots,z_r)
&=
\sum_{\substack{I\subseteq[n],\,|I|=n_0\\ J\subseteq[m],\,|J|=m_0}}
f_{Z\mid E_{I,J}}(z_1,\dots,z_r)\,
\mathbb P(E_{I,J}\mid n_0,m_0) \\
&=
\sum_{\substack{I\subseteq[n],\,|I|=n_0\\ J\subseteq[m],\,|J|=m_0}}
f_{Z\mid E_{I,J}}(z_{\pi(1)},\dots,z_{\pi(r)})\,
\mathbb P(E_{I,J}\mid n_0,m_0) \\
&=
f_{Z\mid n_0,m_0}(z_{\pi(1)},\dots,z_{\pi(r)}).
\end{aligned}
\]
Hence \(Z\mid(n_0,m_0)\) is exchangeable.

Therefore, conditional on \((n_0,m_0)\), the null calibration scores and the true-null test scores satisfy the exchangeability assumption required in the split-conformal multiple-testing framework. 
\citet{bates2023testing} show that conformal \(p\)-values constructed from exchangeable scores are PRDS on the subset of true nulls. Applying that result conditionally on \((n_0,m_0)\), we conclude that the conformal \(p\)-values are PRDS on the subset \(\mathcal H_0\).

Algorithm~\ref{alg:CL_conformal} applies the BH procedure at level
\[
\frac{n+1}{n_0+1}\alpha.
\]
Therefore, the standard BH guarantee under super-uniform null \(p\)-values and PRDS dependence yields
\[
\mathbb E\!\left[\mathrm{FLP}\mid n_0,m_0\right]
\le
\frac{m_0}{m}\cdot \frac{n+1}{n_0+1}\alpha.
\]
This completes the proof.
\end{proof}

\begin{proof}[Proof of Theorem \ref{thm:CL_FLR_iid}]

Under the i.i.d. assumption of Theorem~\ref{thm:CL_FLR_iid}, the calibration samples \((X_i, Y_i)_{i=1}^n\) and the test samples \((X_{n+j}, Y_{n+j})_{j=1}^m\) are independently and identically distributed.
This implies that each hypothesis, whether from the calibration or test set, has an equal probability \( \pi_0 \) of being a true null, where \( \pi_0 \) represents the probability that a prediction is incorrect. 

Consequently, the number of true null hypotheses in the calibration set, denoted \( n_0 \), follows a binomial distribution \( n_0 \sim \text{Binomial}(n, \pi_0) \), and the number of true null hypotheses in the test set, denoted \( m_0 \), follows \( m_0 \sim \text{Binomial}(m, \pi_0) \). The independence across the calibration and test sets arises from the i.i.d. structure, ensuring that \( n_0 \) and \( m_0 \) are independent random variables.

Using the law of total expectation, we express FLR as
\begin{equation}
    \begin{aligned}
        \text{FLR} =& \mathbb{E}[\text{FLP}]\\
        =& \mathbb{E}[\mathbb{E}[\text{FLP} \mid n_0, m_0]]\\
        \leq& \mathbb{E}\left[ \frac{m_0}{m} \frac{1 + n}{1 + n_0} \alpha\right] \quad \text{by Lemma \ref{lem:FLPbound}}\\
        =& \mathbb{E}\left[\frac{m_0}{m}\right] \cdot\mathbb{E}\left[\frac{1 + n}{1 + n_0}\right] \cdot \alpha \quad \text{since } n_0 \text{ and } m_0 \text{  are independent}\\
        =& \pi_0\cdot\mathbb{E}\left[\frac{1 + n}{1 + n_0}\right] \cdot \alpha \quad \text{since }  m_0 \sim \text{Binomial}(m, \pi_0) 
    \end{aligned}
\end{equation}

Now it suffices to show that \( \pi_0 \cdot \mathbb{E}\left[\frac{1 + n}{1 + n_0}\right] = [1 - (1 - \pi_0)^{n+1}]\).

Since \(n_0\) has probability mass function
\[
\mathbb{P}(n_0 = k) = \binom{n}{k} \pi_0^k (1-\pi_0)^{n-k}, \quad k = 0,1,\dots,n,
\]
we compute

\begin{equation}
    \begin{aligned}
        \pi_0 \cdot \mathbb{E}\left[\frac{1+n}{1+n_0}\right] =&\pi_0\cdot \sum_{k=0}^n \frac{1+n}{1+k} \binom{n}{k} \pi_0^{k} (1-\pi_0)^{n-k}\\
        =& \sum_{k=0}^n \frac{(n+1)!}{(k+1)!(n-k)!} \, \pi_0^{k+1} (1-\pi_0)^{n-k}\\
        =& \sum_{k=0}^n \binom{n+1}{k+1} \, \pi_0^{k+1} (1-\pi_0)^{n-k}\\
        =& \sum_{l=1}^{n+1} \binom{n+1}{l} \pi_0^l (1-\pi_0)^{n+1-l} \quad \text{by letting }l=k+1 \\
        =& \sum_{l=0}^{n+1} \binom{n+1}{l} \pi_0^l (1-\pi_0)^{n+1-l} - \binom{n+1}{0} \pi_0^0 (1-\pi_0)^{n+1-0}\\
        =& \sum_{l=0}^{n+1}\mathbb{P}(X=l) - (1-\pi_0)^{n+1} \quad \text{where } X \sim \text{Binomial}(n+1, \pi_0)\\
        =& 1 - (1-\pi_0)^{n+1}
    \end{aligned}
\end{equation}

This completes the proof, establishing the desired bound on the FLR.

\end{proof}

\subsection{Proof of Theorem~\ref{thm:asymptotic_power}} \label{subsection:asymptotic_optimality}
\begin{proof}[Proof of Theorem~\ref{thm:asymptotic_power}]
Recall that
\[
    G_0(t):=\mathbb P(\hat p_j\le t\mid H_j^0),
    \qquad
    G_1(t):=\mathbb P(\hat p_j\le t\mid H_j^1).
\]
Let $\pi_0:=\mathbb{P}(H_j^0\text{ is true})$ denote the proportion of incorrect auto-generated labels, and let $\pi_1:=1-\pi_0=\mathbb{P}(H_j^1\text{ is true})$ denote the proportion of correct auto-generated labels.
Let
\[
   G(t):=\pi_0G_0(t)+\pi_1G_1(t)
\]
denote the limiting marginal distribution function of the conformal
\(p\)-values among the test samples. For any fixed threshold \(t\in[0,1]\), define
\[
    V_m(t):=\sum_{j=1}^m \mathbf 1\{\hat p_j\le t,\ H_j^0\text{ is true}\},
    \qquad
    T_m(t):=\sum_{j=1}^m \mathbf 1\{\hat p_j\le t,\ H_j^1\text{ is true}\},
\]
and
\[
    R_m(t):=\sum_{j=1}^m \mathbf 1\{\hat p_j\le t\}
    =V_m(t)+T_m(t).
\]
By the law of large numbers,
\[
    \frac{V_m(t)}{m}\to \pi_0G_0(t),\qquad
    \frac{T_m(t)}{m}\to \pi_1G_1(t),\qquad
    \frac{R_m(t)}{m}\to G(t)
\]
in probability. Hence, for any fixed \(t\) with \(G(t)>0\),
\[
    \mathrm{FLP}_m(t)
    =
    \frac{V_m(t)}{R_m(t)\vee 1}
    =
    \frac{V_m(t)/m}{(R_m(t)\vee 1)/m}
    \to
    \frac{\pi_0G_0(t)}{G(t)}
    =
    \frac{\pi_0G_0(t)}
    {\pi_0G_0(t)+\pi_1G_1(t)}.
\]
Here \(R_m(t)\to\infty\) in probability whenever \(G(t)>0\). Therefore,
\[
    \mathrm{FLR}_\infty(t)
    =
    \frac{\pi_0G_0(t)}
    {\pi_0G_0(t)+\pi_1G_1(t)},
\]
with the convention that \(0/0=0\).
Similarly, if
\[
    m_1:=\sum_{j=1}^m \mathbf 1\{H_j^1\},
\]
then \(m_1/m\to \pi_1\), and hence
\[
    \mathrm{Power}_m(t)
    =
    \frac{T_m(t)}{m_1\vee 1}
    \to
    G_1(t).
\] Thus
\[
    \mathrm{Power}_\infty(t)=G_1(t).
\]

We next characterize the limiting Conformal Labeling threshold. Under the
i.i.d. assumption,
\[
    \frac{n+1}{n_0+1}\to \frac{1}{\pi_0}
\]
in probability. Therefore, the BH level used by Conformal Labeling satisfies
\[
    q_n:=\frac{n+1}{n_0+1}\alpha
    \to
    \frac{\alpha}{\pi_0}
\]
in probability. By the BH rule,
\[
    \hat k
    =
    \max\left\{
        k:\hat p_{(k)}\le q_n\frac{k}{m}
    \right\},
    \qquad
    \hat t_{\mathrm{CL}}=\hat p_{(\hat k)}.
\]
Using \(R_m(\hat p_{(k)})=k\) after randomized tie-breaking, the same rule can
be written as
\[
    \hat t_{\mathrm{CL}}
    =
    \max\left\{
        t\in\{\hat p_1,\ldots,\hat p_m\}:
        t\le q_n\frac{R_m(t)}{m}
    \right\}.
\]

Now we use the continuity assumption. Since \(\mathcal S(X)\) has continuous
conditional distributions under both \(H_j^0\) and \(H_j^1\), the induced
distribution functions \(G_0\) and \(G_1\) are continuous. Hence
\[
    G(t)=\pi_0G_0(t)+\pi_1G_1(t)
\]
is also continuous. Let \(\widehat G_m(t):=R_m(t)/m\). By the
Glivenko--Cantelli theorem,
\[
    \sup_{t\in[0,1]}|\widehat G_m(t)-G(t)|\to 0
\]
in probability. Together with \(q_n\to \alpha/\pi_0\) and the continuity of
\(G\), this implies that the empirical BH self-consistency condition
\[
    t\le q_n\widehat G_m(t)
\]
converges to its population counterpart
\[
    t\le \frac{\alpha}{\pi_0}G(t).
\]
Consequently,
\[
    \hat t_{\mathrm{CL}}\to t_{\mathrm{CL}},
    \qquad
    t_{\mathrm{CL}}
    :=
    \sup\left\{
        t\in[0,1]:
        t\le \frac{\alpha}{\pi_0}G(t)
    \right\}.
\]
The continuity assumption is used here to rule out jumps in the limiting
distribution \(G\), which ensures that the empirical threshold $\hat{t}_{\mathrm{CL}}$
converges to the fixed population threshold \(t_{\mathrm{CL}}\).

Since \(\hat t_{\mathrm{CL}}\to t_{\mathrm{CL}}\) and \(G_0,G_1\) are
continuous, the fixed-threshold limits above extend to the data-dependent
CL threshold. Therefore,
\[
    \mathrm{FLR}^{\mathrm{CL}}_\infty
    =
    \frac{\pi_0G_0(t_{\mathrm{CL}})}
    {\pi_0G_0(t_{\mathrm{CL}})+\pi_1G_1(t_{\mathrm{CL}})},
    \qquad
    \mathrm{Power}^{\mathrm{CL}}_\infty
    =
    G_1(t_{\mathrm{CL}}).
\]

It remains to prove the optimality claim. 
Because \(\mathcal S(X)\) has continuous conditional distributions under both \(H_j^0\) and \(H_j^1\), the uncertainty scores $\mathcal{S}(X)$ have no
ties almost surely and thus conformal \(p\)-values follow a uniform distribution under the null. 
Thus,
\[
    G_0(t)=t,\qquad t\in[0,1].
\]
Therefore, for any thresholding rule
\(\mathcal R_t=\{j:\hat p_j\le t\}\),
\[
    \mathrm{FLR}_\infty(t)
    =
    \frac{\pi_0G_0(t)}{G(t)}
    =
    \frac{\pi_0 t}{G(t)}.
\]
If \(\mathcal R_t\) has asymptotic FLR at most \(\alpha\), then
\[
    \frac{\pi_0 t}{G(t)}\le \alpha,
\]
which is equivalent to
\[
    t\le \frac{\alpha}{\pi_0}G(t).
\]
Thus every threshold satisfying the asymptotic FLR constraint is feasible
for the limiting BH inequality. Since \(t_{\mathrm{CL}}\) is the largest
feasible threshold, we have \(t\le t_{\mathrm{CL}}\). Because \(G_1(t)\) is
nondecreasing,
\[
    \mathrm{Power}_\infty(\mathcal R_t)
    =
    G_1(t)
    \le
    G_1(t_{\mathrm{CL}})
    =
    \mathrm{Power}^{\mathrm{CL}}_\infty.
\]
This proves the claimed asymptotic power optimality among all
\(p\)-value thresholding rules satisfying the asymptotic FLR constraint.
\end{proof}

\subsection{Proof of Theorem~\ref{thm:oracle_score}} 
\label{subsection:optimal_score}

For a selection rule \(\epsilon_j\in\{0,1\}\), let
\[
    R_m:=\sum_{j=1}^m \epsilon_j,\qquad
    V_m:=\sum_{j=1}^m \epsilon_j\mathbf 1\{H_j^0\text{ is true}\},
    \qquad
    T_m:=\sum_{j=1}^m \epsilon_j\mathbf 1\{H_j^1\text{ is true}\}.
\]
We define
\[
    \mathrm{FLR}:=\mathbb E\left[\frac{V_m}{R_m\vee 1}\right],
    \qquad
    \mathrm{mFLR}:=\frac{\mathbb E[V_m]}{\mathbb E[R_m]}.
\]
FLR is the expectation of the ratio $\frac{V_m}{R_m\vee 1}$, whereas mFLR is the ratio of the expectations $\mathbb E[V_m]$ and $\mathbb E[R_m]$.

\paragraph{Proof sketch.}
We first prove that, under mFLR control, the most powerful rule selects
instances with the smallest posterior null probability
\[
    S^\star(X):=\mathbb P(H^0\text{ is true}\mid X).
\]
We then connect this mFLR oracle result to Conformal Labeling.
In finite samples, Conformal Labeling uses a data-dependent cutoff chosen by a step-up procedure.
Asymptotically, by Theorem~\ref{thm:asymptotic_power}, this cutoff converges to a deterministic population-level cutoff as the calibration and test sample sizes grow.
Therefore, asymptotically, Conformal Labeling is equivalent to a
fixed-threshold selection rule at the population level.

For such population-level fixed-threshold rules, the mFLR and FLR constraints are asymptotically equivalent~\citep{genovese2002operating, sun2007oracle}.
Intuitively, by the law of large numbers, $\frac{V_m}{m}$ and $\frac{R_m}{m}$ converge to $\mathbb E\big[\frac{V_m}{m}\big]$ and $\mathbb E\big[\frac{R_m}{m}\big]$, respectively, so the expectation of the ratio, $\mathrm{FLR}=\mathbb E\big[\frac{V_m}{R_m\vee 1}\big]$, approaches the ratio of expectations, $\mathrm{mFLR}=\frac{\mathbb E[V_m]}{\mathbb E[R_m]}$.
Consequently, the score that is power-optimal under mFLR control is also asymptotically power-optimal for Conformal Labeling under FLR control.
Thus, the oracle uncertainty score for Conformal Labeling is
\[
    S^\star(X)=\mathbb P(H^0\text{ is true}\mid X),
\]
up to any monotone transformation.

\begin{proof}[Proof of Theorem~\ref{thm:oracle_score}]
Let
\[
    Z_j:=\mathbf 1\{H_j^0 \text{ is true}\}
\]
denote the indicator that the \(n+j\)-th auto-generated label is incorrect. Define
\[
    S_j^\star
    :=
    S^\star(X_{n+j})
    =
    \mathbb P(H_j^0 \text{ is true}\mid X_{n+j})
    =
    \mathbb E[Z_j\mid X_{n+j}].
\]
Thus \(S_j^\star\) is the posterior probability that selecting the \(j\)-th
auto-generated label is incorrect.

Consider any measurable score-based selection rule
\[
    \epsilon_j=\epsilon(X_{n+j})\in\{0,1\},
\]
where \(\epsilon_j=1\) means that the \(j\)-th auto-generated label is selected. By
the tower property,
\[
    \mathbb E[\epsilon_j Z_j]
    =
    \mathbb E\!\left[
        \epsilon_j\,\mathbb E[Z_j\mid X_{n+j}]
    \right]
    =
    \mathbb E[\epsilon_j S^\star(X_{n+j})].
\]
Therefore, by exchangeability,
\[
    \mathrm{mFLR}(\epsilon)
    =
    \frac{\mathbb E[V_m]}{\mathbb E[R_m]}
    =
    \frac{\mathbb E[\epsilon(X)S^\star(X)]}
    {\mathbb E[\epsilon(X)]}.
\]
Similarly, the expected number of true discoveries is
\[
    \mathbb E[T_m]
    =
    m\,\mathbb E[\epsilon(X)(1-S^\star(X))].
\]
Thus maximizing power under the mFLR constraint is equivalent to solving
\[
    \max_{\epsilon:\mathcal X\to\{0,1\}}
    \mathbb E[\epsilon(X)(1-S^\star(X))]
    \quad
    \mathrm{s.t.}
    \quad
    \mathbb E[\epsilon(X)S^\star(X)]
    \le
    \alpha\,\mathbb E[\epsilon(X)].
\]
Equivalently, the constraint is
\[
    \mathbb E[\epsilon(X)(S^\star(X)-\alpha)]\le 0.
\]

We now solve the oracle constrained problem. For a multiplier
\(\lambda\ge 0\), define the Lagrangian
\[
\begin{aligned}
    \mathcal L(\epsilon;\lambda)
    &=
    \mathbb E[\epsilon(X)(1-S^\star(X))]
    -
    \lambda \mathbb E[\epsilon(X)(S^\star(X)-\alpha)] \\
    &=
    \mathbb E\!\left[
        \epsilon(X)\{1-S^\star(X)-\lambda(S^\star(X)-\alpha)\}
    \right].
\end{aligned}
\]
Let
\[
    c_\lambda(x)
    :=
    1-S^\star(x)-\lambda(S^\star(x)-\alpha).
\]
Then
\[
    \mathcal L(\epsilon;\lambda)
    =
    \mathbb E[\epsilon(X)c_\lambda(X)].
\]
For each fixed \(x\), since \(\epsilon(x)\in\{0,1\}\), choosing
\(\epsilon(x)=0\) gives contribution \(0\), while choosing \(\epsilon(x)=1\)
gives contribution \(c_\lambda(x)\). Therefore, the pointwise maximizer is
\[
    \epsilon^\lambda(x)
    =
    \mathbf 1\{c_\lambda(x)\ge 0\}.
\]
The condition \(c_\lambda(x)\ge 0\) is equivalent to
\[
\begin{aligned}
    1-S^\star(x)-\lambda(S^\star(x)-\alpha) &\ge 0,\\
    1-S^\star(x)-\lambda S^\star(x)+\lambda\alpha &\ge 0,\\
    1+\lambda\alpha-(1+\lambda)S^\star(x) &\ge 0,\\
    S^\star(x) &\le \frac{1+\lambda\alpha}{1+\lambda},
\end{aligned}
\]
where the last step uses \(1+\lambda>0\). Hence, for each \(\lambda\ge 0\),
the Lagrangian optimizer has the thresholding form
\[
    \epsilon^\lambda(x)
    =
    \mathbf 1\left\{
        S^\star(x)\le \frac{1+\lambda\alpha}{1+\lambda}
    \right\}.
\]
Choosing the largest threshold satisfying the mFLR constraint gives the oracle
rule
\[
    \epsilon_t^\star(X)=\mathbf 1\{S^\star(X)\le t\}.
\]
Therefore, the mFLR-oracle rule ranks instances by increasing posterior null
probability
\[
    S^\star(X)=\mathbb P(H^0\text{ is true}\mid X),
\]
up to any monotone transformation.

In finite samples, Conformal Labeling does not use a fixed cutoff. Given an
uncertainty score, it first orders the test instances and then applies a
BH-type step-up rule to choose the largest data-dependent cutoff satisfying the empirical FLR constraint. Therefore, the score affects the procedure only
through the induced ordering of the test instances.

By Theorem~\ref{thm:asymptotic_power}, as \(n,m\to\infty\), this random cutoff converges to a deterministic population cutoff. 
Thus, asymptotically, Conformal Labeling behaves like a population thresholding rule: it selects all instances whose score falls below the limiting cutoff. 
For such limiting thresholding rules, \(\mathrm{mFLR}\) and \(\mathrm{FLR}\) constraints are asymptotically equivalent~\citep{genovese2002operating}. 
Thus, the score that is power-optimal under mFLR control is also asymptotically power-optimal for Conformal Labeling under FLR control.
Consequently,
\[
    S^\star(X_{n+j})
    =
    \mathbb P(H_j^0 \text{ is true}\mid X_{n+j})
\]
is the asymptotically power-optimal uncertainty score for Conformal
Labeling.
\end{proof}

\section{Overview of Different Uncertainty Score Functions in Figure \ref{fig:score_comparisons}}\label{section:score}

In this section, we provide an overview of three representative uncertainty score functions that are widely used in misclassification detection: Maximum Softmax Probability (MSP) \citep{hendrycks2016baseline}, the energy-based score \citep{liu2020energy}, and the DOCTOR-$\alpha$ score \citep{granese2021doctor}. 
Each of these functions captures predictive uncertainty from a different perspective.

\paragraph{Maximum Softmax Probability (MSP).} 
The Maximum Softmax Probability (MSP) baseline \citep{hendrycks2016baseline} proposes to use the confidence of the AI model as an uncertainty score.
\[
S_{\mathrm{MSP}}(x) = 1-  \max_{y \in \mathcal{Y}} p_y(x),
\]
where $p_y(x)$ is the softmax probability assigned to class $y$ for input $x$. 
The key idea is straightforward: if the model assigns a high probability to its most likely class, the prediction is considered confident. 
Although MSP is simple and effective, it has notable limitations: it only reflects the confidence in the top-1 prediction and ignores the structure of the remaining probability distribution, which may contain useful information about uncertainty.

\paragraph{Energy-based score.} 
The energy-based score \citep{liu2020energy} is defined as
\[
S_{\mathrm{Energy}}(x) = 1 - \log \sum_{y \in \mathcal{Y}} \exp\!\left(f_y(x)\right),
\]
where $f_y(x)$ denotes the logit value for class $y$. 
This score is derived from the concept of energy in statistical physics and leverages the log-sum-exp operator over all logits. 
Unlike MSP, which only considers the maximum probability, the energy score integrates information from the full logit vector, thereby providing a smoother and more informative uncertainty measure. 

\paragraph{DOCTOR-$\alpha$ score.} 
The DOCTOR-$\alpha$ score \citep{granese2021doctor} is defined as
\[
S_{\alpha}(x) =  1-\sum_{y \in \mathcal{Y}} p_y(x)^2,
\]
where $p_y(x)$ denotes the softmax probability for class $y$. 
This score is inspired by information-theoretic measures of uncertainty, as it is closely related to the quadratic R\'enyi entropy. 
The intuition is that if the predictive distribution is sharp (i.e., one class has probability close to one), then $\sum_y p_y(x)^2$ will be large, indicating low uncertainty. 
Conversely, if the distribution is flat (i.e., the model is uncertain and spreads probability mass across many classes), then $S_{\alpha}(x)$ will be large. 
Compared to MSP, the DOCTOR-$\alpha$ score leverages information from the entire probability distribution rather than only the top prediction, making it a richer measure of uncertainty for misclassification detection.

\section{Compared baselines}
\label{sec:baselines}
In this work, we compare Conformal Labeling's selection procedures with BH, Storey-BH, and Quantile-BH procedures, which target the same FLR objective.
In addition, we compare Conformal Labeling against four baselines:
selection with guaranteed risk control (SGR)~\citep{geifman2017selective}, 
PAC Labeling~\citep{candes2024probablyapproximately}, 
confidence thresholding, and the FLR-search heuristic.
SGR applies selective classification with risk control to our labeling setting; it searches over a grid of confidence thresholds and uses a union bound to provide a high-probability guarantee on the marginal selective risk.
PAC Labeling selectively combines auto-generated labels with expert annotations and provides a high-probability guarantee on the overall error rate of the final labeled dataset.
Confidence thresholding directly accepts auto-generated labels whose uncertainty scores satisfy $\mathcal{S}_{n+j}\le \alpha$.
The FLR-search heuristic selects the largest uncertainty threshold whose empirical FLR on the calibration set does not exceed the target level $\alpha$.

We provide detailed descriptions of these baselines below.

\subsection{BH Procedure and its Adaptive Variants} \label{section:bh_var}
Consider testing $m$ null hypotheses $H^{0}_1, \ldots, H^{0}_m$ based on their corresponding $p$-values $\{p_1, p_2, \ldots, p_m\}$. 
For a true null hypothesis $H^0_j$, the corresponding $p$-value $p_j$ is a random variable that is super-uniform on $[0,1]$ under the null hypothesis.
Formally, for any $u \in [0,1]$,
\[
\mathbb{P}(p_j \leq u \mid H^{0}_j \text{ is true}) \leq u.
\]

Define $p_{(j)}$ as the $j$-th smallest $p$-value among a set of $p$-values $\{p_1, p_2, \ldots, p_m\}$.
Given a set of $p$-values $\{p_1, p_2, \ldots, p_m\}$, the BH algorithm returns $\mathcal{R}=\{j\in\{1, \ldots,m\}:p_{j}\leq\frac{\alpha j^*}{m}\}$, where $\alpha$ is the target FLR level and 
\[
j^*=\text{max}\{j\in\{1,\ldots,m\}:p_{(j)} \leq\frac{\alpha j}{m}\}
\]
When the null $p$-values $\{p_j :j \in \mathcal{H}_0\}$ are independent, the BH procedure is proved to control the FLR at level $\pi_0\alpha$ in finite samples \citep{benjamini1995controlling}, where $\pi_0=\frac{|\mathcal{H}_0|}{m}$ is the proportion of true nulls.
The independence assumption can be further relaxed to the PRDS condition \citep{benjamini2001control}. 

If $\pi_0$ is small, the FLR control will be overly conservative.
For example, if a classifier achieves an 80\% accuracy in the unlabeled test dataset, then $\pi_0 = 0.2$-which leads to overly conservative FLR control.
When $\pi_0$ is known, we can apply the BH procedure at level $\frac{\alpha}{\pi_0}$ to close the gap.
In practice,  $\pi_0$ is typically unknown.
Several adaptive BH procedures attempt to address this issue by estimating $\pi_0$ and adjusting the target FLR level $\alpha$ accordingly.
These procedures are often called the $\pi_0$-adaptive versions of the BH algorithm.
Two commonly used estimators are Storey-BH \citep{storey2002direct} and Quantile-BH \citep{benjamini2006adaptive}:

\[\hat{\pi}_0^{Storey}(\lambda)=\frac{1 + \sum_{i=1}^m\mathbf{1}\{p_i\geq \lambda\}}{m(1-\lambda)}, \lambda \in (0,1)\]
\[\hat{\pi}_0^{Quant}(k_0)=\frac{m-k_0+1}{m(1-p_{(k_0)})}, k_0 \in \{1,\dots,m\},\]
where $\lambda$ and $k_0$ are hyperparameters determined by users.
In this work, we use the bootstrap method of~\citet{storey2002direct} to determine these hyperparameters; details are provided in Appendix~\ref{section:hyperparameter_selection}.

\subsection{Selection with guaranteed risk control (SGR)}
Let $f$ be a fixed classifier and let $c:\mathcal{X}\to\mathbb{R}$ be a confidence score, where larger values indicate more reliable predictions.
For a threshold $\theta$, define the selection rule
\[
g_\theta(x)=\mathbf{1}\{c(x)\ge \theta\},
\]
so that predictions are made only on instances with confidence at least $\theta$.
For a calibration set $\{(x_i,y_i)\}_{i=1}^n$, the empirical selective risk at threshold $\theta$ is
\[
\widehat{R}(\theta)
=
\frac{
\sum_{i=1}^n \mathbf{1}\{g_\theta(x_i)=1,\; f(x_i)\neq y_i\}
}{
\max\left\{\sum_{i=1}^n \mathbf{1}\{g_\theta(x_i)=1\},\,1\right\}}
.\]
The SGR algorithm searches over candidate thresholds (ordered from high to low confidence) and, for each threshold, computes a high-probability upper bound on the true selective risk based on the calibration data.
The output threshold $\theta^\ast$ is the largest threshold whose risk bound does not exceed the target risk level $\alpha$.
At test time, an instance is selected if $c(X_i)\ge \theta^\ast$.

\begin{theorem}[Selective risk guarantee of SGR]
Fix a target risk level $\alpha\in(0,1)$ and a confidence parameter $\epsilon\in(0,1)$.
Let $\theta^\ast$ be the threshold returned by the SGR algorithm using a calibration set of size $n$.
Then, with probability at least $1-\epsilon$ over the draw of the calibration set, the resulting selective classifier $(f,g_{\theta^\ast})$ satisfies
\[
R(f,g_{\theta^\ast}) \;\le\; \alpha,
\]
where $R(f,g_{\theta^\ast})$ denotes the true selective risk.
\end{theorem}

In the labeling task, the risk function is defined as
\[
R = \frac{\mathbb{E}\left[\sum_{j=1}^m \mathbf{1}\{Y_{n+j} \neq \hat{Y}_{n+j} \text{ and } j \in \mathcal{R}\}\right]}{\mathbb{E}\left[\sum_{j=1}^m \mathbf{1}\{j \in \mathcal{R}\}\right]},
\]
which is also known as the marginal false labeling rate (mFLR) in hypothesis testing.
In the literature, FLR control is generally preferred over mFLR control.
We provide a discussion of mFLR and FLR control in Appendix~\ref{subsec:mFLR_vs_FLR}.

\subsection{Probably Approximately Correct (PAC) Labels}
PAC labeling~\citep{candes2024probablyapproximately} is a selective annotation framework that combines inexpensive auto-generated labels with expert annotations while providing a dataset-level accuracy guarantee.
Given an unlabeled dataset $\{X_i\}_{i=1}^n$, an AI model provides predicted labels $\hat{Y}_i=f(X_i)$ and uncertainty scores $U_i$.
For an uncertainty threshold $u$, PAC labeling uses AI labels for low-uncertainty instances and expert labels for high-uncertainty instances:
\[
\widetilde{Y}_i(u)
=
Y_i \mathbf{1}\{U_i \ge u\}
+
\hat{Y}_i \mathbf{1}\{U_i < u\}.
\]
The method estimates an uncertainty threshold $\hat{u}$ using a small number of expert labels and a high-probability upper confidence bound on the labeling error.
At test time, instances with $U_i \ge \hat{u}$ are escalated to expert annotators, while the remaining instances are labeled by the AI model.

\begin{theorem}[Dataset-level accuracy guarantee of PAC labels]
Fix a target error level $\alpha\in(0,1)$ and a failure probability $\epsilon\in(0,1)$.\footnote{
We use $\alpha$ to denote the target error level and $\epsilon$ to denote the failure probability throughout this paper for notational consistency. This convention differs from \citet{candes2024probablyapproximately}, where $\epsilon$ denotes the target error level and $\alpha$ denotes the failure probability.
}
Let $\hat{u}$ be the threshold returned by the PAC labeling algorithm, and let
\[
\widetilde{Y}_i
=
Y_i \mathbf{1}\{U_i \ge \hat{u}\}
+
\hat{Y}_i \mathbf{1}\{U_i < \hat{u}\}
\]
be the resulting labels.
If the mean upper-bound subroutine used by PAC labeling is valid at level $1-\epsilon$, then, with probability at least $1-\epsilon$ over the randomness of the labeling algorithm,
\[
\frac{1}{n}\sum_{i=1}^n \ell(Y_i,\widetilde{Y}_i)
\le
\alpha .
\]
In particular, for the zero-one loss $\ell(Y_i,\widetilde{Y}_i)=\mathbf{1}\{Y_i\neq \widetilde{Y}_i\}$,
\[
\frac{1}{n}\sum_{i=1}^n \mathbf{1}\{Y_i\neq \widetilde{Y}_i\}
\le
\alpha .
\]
\end{theorem}


\subsection{Confidence thresholding}
Let $S(X)$ be an uncertainty score, where smaller values indicate higher confidence.
Given a target error level $\alpha$, an instance is selected if its uncertainty score satisfies
\[
S(X) \leq \alpha .
\]

\subsection{FLR-search heuristic baseline}
Let $S(X)$ be an uncertainty score, where smaller values indicate higher confidence.
For an uncertainty threshold $t$, we define the empirical FLR on the calibration set as
\[
\widehat{\mathrm{FLR}}(t) 
= \frac{ 
    \sum_{i=1}^n \mathbf{1}\{S(x_i) \leq t,\; Y_i \neq \hat{Y}_i\} 
}{
    \max\left\{ \sum_{i=1}^n \mathbf{1}\{S(x_i) \leq t\},\; 1 \right\}
}.
\]
The threshold is selected by searching over uncertainty values from the smallest to the largest.
Specifically, we start from the smallest possible threshold and progressively increase $t$ until the empirical FLR on the calibration set exceeds the target level $\alpha$, and then choose the largest threshold whose empirical FLR remains below $\alpha$:
\[
t^\ast(\alpha)
=
\max\{\, t : \widehat{\mathrm{FLR}}_{\mathrm{cal}}(t) \le \alpha \,\}.
\]
Then a test instance is selected if its uncertainty score satisfies
\[
S(X_i) \leq t^\ast(\alpha).
\]

\section{Hyperparameter selection for BH adaptive variants}\label{section:hyperparameter_selection}
Both Storey-BH and Quantile-BH require careful hyperparameter selection-$\lambda$ for Storey-BH and $k_0$ for Quantile-BH-as this choice significantly impacts their performance. 
Following \citet{storey2002direct}, we employ a bootstrap-based method to select the optimal $\lambda$ or $k_0$ for these procedures.
Further details can be found in Section 9 of \citet{storey2002direct}.
The algorithm proceeds as follows:

\begin{enumerate}
    \item Define a grid $R$ for the hyperparameter, i.e., $R = \{0.1, 0.2, \dots, 0.9\}$ for Storey-BH.
    
    \item For each $\lambda \in R$, compute:
    \begin{equation}
        \widehat{\mathrm{pFLR}}_\lambda(\gamma) = \frac{\widehat{\pi}_0(\lambda)\gamma}{\widehat{\Pr}(p \leq \gamma)\{1 - (1-\gamma)^m\}}
    \end{equation}
    where $\widehat{\pi}_0(\lambda) = \widehat{\pi}_0^{Storey}(\lambda) \text{ or } \widehat{\pi}_0^{Quant}(\lambda)$ and $\widehat{\Pr}(p \leq \gamma)$ is the empirical estimate of $\Pr(p \leq \gamma)$.
    
    \item Generate $B$ bootstrap replicates $\{p_1^{*,b}, \dots, p_m^{*,b}\}_{b=1}^B$ and compute $\widehat{\mathrm{pFLR}}_\lambda^{*,b}(\gamma)$ for each $b$.
    
    \item Estimate the MSE for each $\lambda$:
    \begin{equation}
        \widehat{\mathrm{MSE}}(\lambda) = \frac{1}{B}\sum_{b=1}^B \left( \widehat{\mathrm{pFLR}}_\lambda^{*,b}(\gamma) - \min_{\lambda' \in R} \widehat{\mathrm{pFLR}}_{\lambda'}(\gamma) \right)^2
    \end{equation}
    
    \item Select $\hat{\lambda} = \operatorname*{arg\,min}_{\lambda \in R} \widehat{\mathrm{MSE}}(\lambda)$
\end{enumerate}

\section{Uncertainty score functions}
\label{sec:oe_score}

For the image classification task, we use the one minus maximum softmax probability (MSP) as the uncertainty score.
For the LLM multi-choice QA task, we also adopt one minus MSP, where the estimated probability of each option is obtained by applying a softmax over the logits corresponding to the option tokens.

\paragraph{Bottom 10\% group uncertainty score~\citep{fu2025deep}.}
For the open-ended generation task, we adopt the bottom 10\% group uncertainty score with a sliding-window size \(w=50\).
Specifically, for each generated token position \(i\), let \(P_i\) denote the model's next-token distribution.
We define the token-level uncertainty as
\[
U_i = \frac{1}{k}\sum_{j=1}^{k} \log P_i(j),
\]
where \(P_i(j)\) denotes the probability of the \(j\)-th most likely token at position \(i\).
In our implementation, we set \(k=20\).

Given a generated response with \(N\) tokens, we construct overlapping token groups using a sliding window.
For each position \(i\in[N]\), the corresponding group is defined as
\[
G_i=\{\max(1,i-w+1),\ldots,i\},
\]
which consists of the current token and up to \(w-1\) preceding tokens.
We then compute the uncertainty of each group by averaging the token-level uncertainty values within the group:
\[
U_{G_i} = \frac{1}{|G_i|}\sum_{t\in G_i} U_t.
\]
Let \(\mathcal G=\{G_i\}_{i=1}^{N}\) denote the collection of all sliding-window groups.
The bottom 10\% group uncertainty is defined as the average uncertainty over the groups with the highest 10\% group uncertainty values:
\[
U_{\mathrm{bottom}\text{-}10}
=
\frac{1}{|\mathcal G_b|}
\sum_{G_i\in \mathcal G_b} U_{G_i},
\]
where \(\mathcal G_b\) denotes the set of groups with the highest \(\max(1,\lfloor 0.1N\rfloor)\) group uncertainty values.
This score focuses on the least reliable local segments of a generated response rather than averaging uncertainty over the entire response.
We use
\[
S(x)=U_{\mathrm{bottom}\text{-}10}(x)
\]
as the uncertainty score for open-ended generation, so that larger values indicate less reliable auto-generated labels.

\paragraph{Self-consistency-based uncertainty.}
We also consider a self-consistency-based uncertainty score.
For each input, we sample multiple responses and extract their final answers.
Let \(N\) be the number of sampled responses and let \(N_{\mathrm{maj}}\) be the number of responses agreeing with the majority answer.
The self-consistency confidence is defined as
\[
C_{\mathrm{SC}}(x)=\frac{N_{\mathrm{maj}}}{N}.
\]
This score measures how strongly the sampled responses agree on the final answer.
We use its complement,
\[
S_{\mathrm{SC}}(x)=1-C_{\mathrm{SC}}(x),
\]
as the corresponding uncertainty score.

\section{Implementation details} \label{section:moreimplementation}
\paragraph{Hyperparameter settings of LLMs.}
\label{para:hyperparameter}
For the open-ended generation task, we configure the decoding parameters as follows. 
For Qwen/Qwen3-4B-Instruct-2507, we set Temperature=0.7, TopP=0.8, TopK=20, and MinP=0.
For Qwen/Qwen3-4B-Thinking-2507, we set Temperature=0.6, TopP=0.95, TopK=20, and MinP=0.
For DeepSeek/DeepSeek-R1-Distill-Qwen-32B, we set Temperature=0.6, TopP=0.95, and MinP=0.
For experiments on the MATH-500 dataset, we additionally apply a presence penalty of 1.0 to avoid repetitive generation. 
All experiments are run on NVIDIA GeForce RTX 4090, NVIDIA L40, and NVIDIA RTX Pro 6000.

\paragraph{Dataset splitting strategy.}
For image labeling and LLM QA tasks, we randomly sample 1,000 instances from the unlabeled dataset as the calibration set.
For open-ended generation tasks, we use dataset-specific splitting strategies, which are detailed in Table \ref{tab:datasets_split}.

\begin{table}[t]
\centering
\caption{Details of dataset splitting settings for open-ended generation task.}
\label{tab:datasets_split}
\begin{tabular}{llccc}
\toprule
Dataset & Dataset Type & Dataset Size & Split Setting & Size \\
\midrule
\multirow{2}{*}{MATH-500} 
& \multirow{2}{*}{Math Reasoning} 
& \multirow{2}{*}{500} 
& Calibration & 300 \\
& & & Test & 200 \\
\midrule
\multirow{2}{*}{MATH-L5} 
& \multirow{2}{*}{Math Reasoning} 
& \multirow{2}{*}{721} 
& Calibration & 300 \\
& & & Test & 421 \\
\midrule
\multirow{2}{*}{HumanEval+} 
& \multirow{2}{*}{Coding Task} 
& \multirow{2}{*}{164} 
& Calibration & 100 \\
& & & Test & 64 \\
\bottomrule
\end{tabular}
\end{table}

\paragraph{Prompts for LLM multi-choice QA task.} 
We adopt the same prompts as in \citet{luo2025your}. 
The prompt is presented in Table~\ref{tab:llmqa_prompt}. 

\begin{table}[h]
\centering
\renewcommand\arraystretch{2}
\setlength{\tabcolsep}{6mm}
\caption{Prompt for LLM QA task}
\begin{adjustbox}{max width=0.8\textwidth, scale=1.25}
\LARGE
\begin{tabular}{cc}
\hline
\textbf{Dataset} & \textbf{Prompt} \\
\hline
LLM QA &
The following are multiple-choice questions. Give ONLY the correct option, no other words or explanation: \\
& [Question] A: [Option 1] B: [Option 2] C: [Option 3] D: [Option 4] Answer: [Mask] \\
\hline
\end{tabular}
\end{adjustbox}
\label{tab:llmqa_prompt}
\end{table}

\section{Extension to Regression Tasks}\label{section:regression}

While our primary focus is on classification, Conformal Labeling can be naturally extended to regression settings.
Consider a loss function \(L(Y, \hat{Y})\) that quantifies prediction error-for instance, the squared error \(L(Y, \hat{Y}) = (Y - \hat{Y})^2\)—alongside a user-specified tolerance level \(\epsilon\).
For each test sample, we define the null and alternative hypotheses as:
\[
H^0_j: L(Y_{n+j}, \hat{Y}_{n+j}) > \epsilon 
\quad \text{versus} \quad 
H^1_j: L(Y_{n+j}, \hat{Y}_{n+j}) \leq \epsilon.
\]

This framework generalizes the classification setting, which corresponds to the special case where \(L(Y, \hat{Y}) = \mathbf{1}\{Y \neq \hat{Y}\}\) and \(\epsilon = 0\).
To apply Conformal Labeling in regression, we need an uncertainty score that reflects the model's predictive uncertainty. 
While uncertainty score functions are straightforward in classification tasks (e.g., \(1 - \max_{y \in \mathcal{Y}} f_y(X)\)), regression requires alternative approaches to obtain an uncertainty score function.
For example, when using LLMs, we can leverage their verbalized confidence or prompt them to output prediction intervals, using the interval width as the uncertainty score.
The complete procedure for Conformal Labeling in the regression task is outlined in Algorithm~\ref{alg:CL_conformal_reg}.

\begin{algorithm}[t!]
\caption{Conformal Labeling for Regression Tasks}
\label{alg:CL_conformal_reg}
\begin{algorithmic}[1]
\REQUIRE Calibration set $\mathcal{D}_{\mathrm{cal}} = \{(X_i, Y_i)\}_{i=1}^{n}$; 
test instances $\{X_{n+j}\}_{j=1}^m$; 
pre-trained predictor $f$; 
loss function $L$; 
loss threshold $\epsilon$; 
target FLR level $\alpha \in (0,1)$; 
nonconformity scores $\{S_i\}_{i=1}^{n+m}$
\ENSURE Selected set $\mathcal{R}$ with FLR control

\STATE \textit{// Step 0: Identify calibration samples exceeding the loss threshold}
\STATE Set
$\mathcal{D}_{\mathrm{cal}}^0
= \{(X_i, Y_i): L(Y_i, \hat{Y}_i) > \epsilon\}_{i=1}^n$,
and let $n_0 = |\mathcal{D}_{\mathrm{cal}}^0|$.

\STATE \textit{// Step 1: Construct conformal $p$-values}
\FOR{$j = 1$ to $m$}
    \STATE Compute the conformal $p$-value $\hat{p}_j$ according to Eq.~\eqref{eq:conformal_p}.
\ENDFOR

\STATE \textit{// Step 2: Step-up thresholding}
\STATE Compute
$j^* = \max \left\{ j : \hat{p}_{(j)} \le \frac{\alpha j (n+1)}{m (n_0+1)} \right\}$,
where $\hat{p}_{(j)}$ is the $j$-th smallest $p$-value.

\STATE \textbf{Return} $\mathcal{R} = \{ j : \hat{p}_j \le \hat{p}_{(j^*)} \}$.
\end{algorithmic}
\end{algorithm}

We evaluate Conformal Labeling on two regression problems: sentiment analysis with GPT-4o and protein structure prediction with AlphaFold. 
We use the data provided by \citet{candes2024probablyapproximately}.
For sentiment analysis, we prompt GPT-4o to output a prediction interval $[a_i, b_i]$ for each target $Y_i$, set the predicted value as $\hat{Y}_i = (a_i+b_i)/2$, and use the interval length $U_i = b_i-a_i$ as the uncertainty score. 
For protein structure prediction, we take experimentally derived structures as ground truth and AlphaFold predictions as $\hat{Y}_i$.
Uncertainty scores are obtained from AlphaFold's internal confidence measure, the average predicted local distance difference test (pLDDT). 
For both experiments, we employ the L2 loss function.

In Table~\ref{tab:regression}, we report the performance of Conformal Labeling on regression tasks with $\alpha=0.1$. 
In all cases, Conformal Labeling consistently controls the realized FLR below the target level, while a larger loss threshold $\epsilon$ leads to higher power. 


\begin{table*}[t]
\caption{\textbf{Performance of Conformal Labeling on regression tasks at $\alpha=0.1$.} 
Results are shown for sentiment analysis (top) and protein folding (bottom) under different values of the tolerance parameter $\epsilon$. 
In all cases, Conformal Labeling controls the realized FLR below the target level.}
\small
\centering
\begin{adjustbox}{center}
\begin{tabular}{l|l|ccc}
\toprule
\multirow{2}{*}{\textbf{Dataset}} & \multirow{2}{*}{\textbf{Metric}} & \multicolumn{3}{c}{\textbf{Method}} \\
& & CL ($\epsilon=0.05$) & CL ($\epsilon=0.06$) & CL ($\epsilon=0.07$) \\
\midrule
\multirow{2}{*}{Sentiment analysis} 
& {FLR (\%)} & 8.48 & 8.60 & 7.89 \\
& {Power (\%)} & 5.04 & 53.14 & 94.92 \\
\midrule
& & CL ($\epsilon=1$) & CL ($\epsilon=4$) & CL ($\epsilon=9$) \\
\midrule
\multirow{2}{*}{Protein folding}
& {FLR (\%)} & 9.67 & 9.90 & 9.04 \\
& {Power (\%)} & 27.24 & 49.73 & 97.90 \\
\bottomrule
\end{tabular}
\end{adjustbox}
\label{tab:regression}
\end{table*}

\section{Extensive Study}
\label{sec:extensive_study}




\subsection{Practical interpretation of FLR control}
\label{subsec:practical_flr}

The FLR guarantee in Theorem~\ref{thm:CL_FLR_iid} is unconditional: the expectation is taken over the randomness of both the calibration and test samples.
Therefore, it does not guarantee FLR control conditional on a realized calibration set, nor does it provide a high-probability bound on the realized false labeling proportion (FLP) of a single run.
Since practitioners typically apply Conformal Labeling once to a fixed dataset, we empirically characterize the distribution of the realized FLP to better understand what users can expect in a single run.
Specifically, we report the exceedance probability $\mathbb{P}(\mathrm{FLP}>\alpha)$ and the 90th percentile of FLP over 1,000 repetitions under different calibration sizes $n$ in two settings: (1) MMLU with Qwen3-32B at $\alpha=0.1$, a moderately accurate model, and (2) MATH-L5 with Qwen3-4B-Thinking-2507 at $\alpha=0.03$, a highly accurate model with an error rate of $4.30\%$.



\paragraph{Empirical Result.}
As shown in Tables~\ref{tab:flp_mmlu} and~\ref{tab:flp_mathl5}, as $n$ increases, the realized FLP becomes increasingly concentrated around $\alpha$: $\mathbb{P}(\mathrm{FLP}>\alpha)$ approaches $50\%$ and the 90th-percentile FLP decreases towards $\alpha$ (e.g., from $21.44\%$ at $n=10$ to $11.44\%$ at $n=1{,}000$ on MMLU). 
Therefore, with a moderate-sized calibration set, users can expect the realized FLP to remain close to the target in magnitude, although it may exceed $\alpha$.

\paragraph{Practical interpretation.}
In a single run, the realized FLP may be slightly above or below $\alpha$, rather than always below it.
With a moderate calibration set, users can expect the realized FLP to stay close to $\alpha$, and any exceedance to be small.

\begin{table}[t]
\centering
\caption{\textbf{Distribution of realized FLP over 1,000 repetitions under different calibration sizes $n$.}
All values are percentages.}
\label{tab:flp_distribution}
\footnotesize
\begin{subtable}[t]{0.40\linewidth}
\centering
\caption{MMLU, Qwen3-32B, $\alpha=0.1$.}
\label{tab:flp_mmlu}
\setlength{\tabcolsep}{5pt}
\begin{tabular}{ccc}
\toprule
$n$ & $\mathbb{P}(\mathrm{FLP}>\alpha)$ & 90th pct.\ FLP \\
\midrule
10    & 42.19 & 21.44 \\
50    & 44.50 & 15.71 \\
100   & 46.00 & 13.85 \\
1,000 & 48.30 & 11.44 \\
\bottomrule
\end{tabular}
\end{subtable}
\hfill
\begin{subtable}[t]{0.56\linewidth}
\centering
\caption{MATH-L5, Qwen3-4B-Thinking-2507, $\alpha=0.03$.}
\label{tab:flp_mathl5}
\setlength{\tabcolsep}{5pt}
\begin{tabular}{cccc}
\toprule
$n$ & $\mathbb{P}(\mathrm{FLP}>\alpha)$ & 90th pct.\ FLP & Power \\
\midrule
10  & 8.60  & 0.00 & 0.40  \\
50  & 36.71 & 4.62 & 50.58 \\
100 & 41.89 & 4.27 & 66.99 \\
\bottomrule
\end{tabular}
\end{subtable}
\end{table}

\subsection{Discussion of distribution shift}
\label{subsection:distributionshift}
We emphasize that the calibration set is sampled i.i.d. from a large unlabeled dataset (a \textbf{transductive} setting), leaving the remaining data for testing. 
Given a large unlabeled dataset, we first annotate a small subset as the calibration set, then apply conformal labeling to guarantee the reliability of selected auto-generated labels.
This process ensures that the calibration and test sets are naturally i.i.d., satisfying the standard assumptions of conformal labeling.
In other words, Conformal Labeling generally does not encounter distribution shift in practice.

While our method cannot provide theoretical guarantees under distribution shift, we conduct an experiment with ResNet-34 on ImageNet and ImageNet-C Brightness to evaluate its empirical performance. In particular, we use ImageNet as the calibration set and ImageNet-C Brightness with varying severities as the testing set.
We present the performance of our method across various testing sets in Table \ref{tab:brightness_results}. The results show that the realized power of our method is getting worse with a higher severity of distribution shift (may be due to the degraded accuracy), while the FLR is relatively insensitive.
This demonstrates that our method is empirically robust to moderate distribution shift.

Besides, we note that prior work \citep{jin2025model} has investigated approaches for handling covariate shifts in conformal novelty detection and conformal selection. 
The same idea can be naturally incorporated into our Conformal Labeling procedure, enabling it to maintain false labeling rate control under covariate shift. We believe it can be an interesting direction for future work if there are some realistic scenarios with distribution shift.

\begin{table}[t]
\centering
\caption{\textbf{Performance of Conformal Labeling under varying severity levels of ImageNet-C Brightness corruption.} 
"No shift" denotes that the test dataset also comes from ImageNet.
Conformal Labeling maintains stable FLR control close to target levels while power decreases with increasing corruption severity.}
\label{tab:brightness_results}
\footnotesize
\setlength{\tabcolsep}{4pt}
\begin{tabular}{lcccccc}
\toprule
\multirow{2}{*}{Test Dataset} & \multirow{2}{*}{Accuracy (\%)} & \multicolumn{2}{c}{$\alpha=0.05$} & \multicolumn{2}{c}{$\alpha=0.1$} \\
\cmidrule(lr){3-4} \cmidrule(lr){5-6}
 & & FLR (\%) & Power (\%) & FLR (\%) & Power (\%) \\
\midrule
No shift & 73.29 & 5.00 & 72.27 & 9.85 & 79.62 \\
Severity 1 & 68.98 & 5.88 & 65.08 & 11.74 & 81.36 \\
Severity 2 & 67.17 & 5.87 & 63.55 & 11.77 & 80.12 \\
Severity 3 & 64.15 & 5.94 & 61.52 & 11.70 & 78.05 \\
Severity 4 & 59.30 & 5.62 & 57.31 & 11.67 & 75.14 \\
Severity 5 & 52.75 & 5.56 & 53.94 & 11.75 & 72.09 \\
\bottomrule
\end{tabular}
\end{table}

\subsection{Comparison of different score functions}
\label{subsection:scorepower}
\paragraph{The choice of score function only affects the power but not the strict FLR control.}
We present the results of three score functions on the ImageNet dataset in Table \ref{tab:score_comparison}. 
While MSP, Energy, and $D_\alpha$ scores lead to vastly different power, all three successfully control the FLR below the target levels (5\% and 10\%).
Specifically, MSP and $D_\alpha$ achieve high power (64\% and 80\% for $\alpha=5\%$ and $10\%$, respectively), whereas Energy yields significantly lower power (11.83\% and 50.58\%).
This demonstrates that a poor score function can incur higher data collection costs by lowering power, but it does not compromise the statistical guarantee of FLR control.

\begin{table}[t]
\centering
\caption{\textbf{Comparison of FLR control and power for MSP, Energy, and $D_\alpha$ score functions}. The experiment is conducted with a ResNet-34 model on ImageNet. \textbf{Bold} indicates the highest power under the target FLR level. All methods maintain FLR close to target levels, while MSP and $D_\alpha$ achieve significantly higher power than the Energy score.}
\label{tab:score_comparison}
\footnotesize
\begin{tabular}{lcccc}
\toprule
\multirow{2}{*}{Method} & \multicolumn{2}{c}{Target FLR = 5\%} & \multicolumn{2}{c}{Target FLR = 10\%} \\
\cmidrule(lr){2-3} \cmidrule(lr){4-5}
 & FLR (\%) & Power (\%) & FLR (\%) & Power (\%) \\
\midrule
MSP & 4.97 & \textbf{63.87} & 9.97 & \textbf{80.01} \\
Energy & 4.86 & 11.83 & 9.94 & 50.58 \\
$D_\alpha$ & 4.95 & 63.80 & 9.92 & 79.74 \\
\bottomrule
\end{tabular}
\end{table}

\subsection{Performance of Conformal Labeling under self-consistency-based uncertainty score}\label{subsec:self-consistency-confidence}
In the main experiments, we use the logits-based bottom-10\% group uncertainty score because it is broadly applicable to real-world open-ended tasks, including coding tasks where self-consistency confidence is not always well defined. 
For math reasoning tasks, however, self-consistency provides a natural uncertainty measure by quantifying the agreement among multiple sampled responses.

In Table~\ref{tab:self_consistency_score}, we compare the labeling performance of Conformal Labeling using the logits-based bottom-10\% group uncertainty score and the self-consistency-based uncertainty score on math reasoning tasks.
Different from the logits-based score, which measures the least confident local segments of a generated response, the self-consistency-based score measures the agreement among multiple sampled responses. 
The results show that both uncertainty scores lead to valid FLR control under the target level $\alpha=0.025$, while the self-consistency-based score achieves substantially higher power. 
For example, on MATH-L5 with Qwen3-4B-Instruct-2507, the self-consistency-based score improves the power from $27.18\%$ to $95.44\%$, while maintaining a comparable FLR of $2.49\%$. 
These results suggest that self-consistency provides a more reliable uncertainty signal for math reasoning tasks, as agreement among multiple sampled answers better reflects the reliability of the final prediction.

\begin{table*}[t]
\centering
\caption{\textbf{Performance of Conformal Labeling under different uncertainty scores on MATH-L5.}
We compare the logits-based bottom-10\% group uncertainty score with the self-consistency-based uncertainty score.
The ``Error'' column reports the error rate of directly using the corresponding AI-generated labels on the entire dataset.
Results are reported at $\alpha=2.5\%$.
All values are percentages.
\textbf{Bold} indicates the best result.
}
\label{tab:self_consistency_score}
\footnotesize

\setlength{\tabcolsep}{8pt}
\renewcommand{\arraystretch}{1.5}

\begin{tabular}{l *{3}{c} *{3}{c}}
\toprule
\multirow{2}{*}{Model}
& \multicolumn{3}{c}{Logits-based Bottom 10\% Group Uncertainty}
& \multicolumn{3}{c}{Self-consistency-based Uncertainty} \\
\cmidrule(lr){2-4} \cmidrule(lr){5-7}
& Error & FLR & Power
& Error & FLR & Power \\
\midrule

Qwen3-4B-Instruct-2507
& 8.87 & 2.27 & 27.18
& 6.24 & 2.49 & \textbf{95.44} \\

Qwen3-4B-Thinking-2507
& 4.30 & 2.43 & 69.26
& 0.78 & 0.79 & \textbf{100.00} \\

\bottomrule
\end{tabular}
\end{table*}

\subsection{Discussion of PAC labeling}
\label{subsec:pac_labeling}

Probably approximately correct (PAC) labeling~\citep{candes2024probablyapproximately} is a rigorous framework for constructing cost-effective and high-quality labeled datasets using pre-trained AI models.
It aims to guarantee that the overall labeling error of the final dataset is small with high probability.
To achieve this goal, PAC labeling combines AI-generated labels with error-free expert labels.
In this way, the errors introduced by AI-generated labels can be counterbalanced by expert annotations in the final labeled dataset to guarantee low overall labeling error.

Our work studies a different objective.
Instead of controlling the overall error of a completed dataset, we aim to select a large subset of AI-generated labels while controlling the FLR within the selected subset.
This objective is particularly relevant when the selected AI labels themselves are used for downstream training or evaluation, and users need a direct reliability guarantee on the accepted auto-generated labels.

This difference in objective also leads to different theoretical guarantees.
PAC labeling provides a high-probability guarantee on the error rate of the entire labeled dataset after combining AI and expert labels.
In contrast, our method provides FLR control for the selected subset of auto-generated labels.
Therefore, our guarantee does not rely on expert annotations for non-selected samples.

Due to the target difference, it is impractical to compare the performance of our method and PAC labeling in a fair manner.
Nevertheless, to provide an informal empirical comparison, we report the FLR and power of both methods under several hyperparameter settings in Table~\ref{tab:pac_labeling_comparison}.
The results show that our method effectively controls the FLR at the target level, whereas PAC labeling may exceed the target FLR because its guarantee concerns a different error criterion.

\begin{table}[t]
\centering
\caption{\textbf{Informal comparison between Conformal Labeling and PAC labeling.}
We report FLR and power on ImageNet and MMLU. PAC labeling is not designed to control FLR; the comparison is included only to illustrate the difference between the two objectives.}
\label{tab:pac_labeling_comparison}
\small
\setlength{\tabcolsep}{6pt}
\renewcommand{\arraystretch}{1.05}
\begin{tabular}{ll l cc}
\toprule
Dataset & Model & Method & FLR (\%) & Power (\%) \\
\midrule
\multirow{4}{*}{ImageNet}
& \multirow{4}{*}{CLIP-ViT-B/32}
& Ours ($\alpha=0.05$) & 4.96 & 25.58 \\
& & Ours ($\alpha=0.1$) & 9.94 & 45.57 \\
& & PAC labeling ($\alpha=0.05$, $\epsilon=0.1$) & 12.94 & 54.30 \\
& & PAC labeling ($\alpha=0.1$, $\epsilon=0.1$) & 18.51 & 67.89 \\
\midrule
\multirow{4}{*}{MMLU}
& \multirow{4}{*}{Qwen3-32B}
& Ours ($\alpha=0.05$) & 4.90 & 61.31 \\
& & Ours ($\alpha=0.1$) & 9.96 & 82.81 \\
& & PAC labeling ($\alpha=0.05$, $\epsilon=0.1$) & 6.93 & 73.00 \\
& & PAC labeling ($\alpha=0.1$, $\epsilon=0.1$) & 11.90 & 86.92 \\
\bottomrule
\end{tabular}
\end{table}



\subsection{Comparison to heuristic baselines}
\label{subsec:heuristic_baseline}

We compare Conformal Labeling with two heuristic baselines.
\textbf{(1) Confidence threshold.}
Given a target FLR level \(\alpha\), an instance is selected if its one minus maximum softmax probability (MSP) satisfies \(\mathcal{S}(x)\le \alpha\).
\textbf{(2) FLR search.}
For image classification, we compare against both confidence thresholding and FLR search on ImageNet with a ResNet-34 classifier.
For open-ended generation tasks, we only compare against FLR search, because the logits-based bottom-10\% group uncertainty score used is not normalized to the interval \([0,1]\), making the simple rule \(\mathcal{S}(x)\le \alpha\) for confidence thresholding not directly meaningful.

\paragraph{Conformal Labeling achieves higher power than heuristic baselines under the same target FLR level.}
Tables~\ref{tab:method_comparison} and~\ref{tab:flr_search_oe} report the results.
On ImageNet, Conformal Labeling achieves the highest power at both target levels, improving over FLR search from \(63.11\%\) to \(72.27\%\) power at \(\alpha=0.05\), and from \(75.90\%\) to \(79.91\%\) at \(\alpha=0.10\).
On open-ended generation tasks, Conformal Labeling also consistently achieves the highest power across all model--dataset pairs, while maintaining FLR close to or below the target level.
For example, on MATH-L5 with DeepSeek-R1-Distill-Qwen-32B, FLR search exceeds the target level with an FLR of \(5.95\%\) and achieves only \(9.19\%\) power, whereas Conformal Labeling controls the FLR at \(4.72\%\) and achieves \(51.43\%\) power.
Similarly, on HumanEval+ with Qwen3-4B-Thinking-2507, FLR search achieves only \(43.07\%\) power, while Conformal Labeling achieves \(100.00\%\) power.
These results show that heuristic threshold search can be both less powerful and less reliable, whereas Conformal Labeling provides substantially higher power together with rigorous FLR-control guarantees.
Notably, we emphasize that the two heuristic baselines cannot provide any rigorous guarantee. We believe this demonstrates the benefits of our method compared to these heuristic baselines.

\begin{table}[t]
\centering
\caption{\textbf{Comparison of different labeling methods on ImageNet with ResNet-34 at target error levels $\alpha=0.05$ and $\alpha=0.1$.} Conformal Labeling achieves the highest power and the tightest FLR control among the three methods.}
\label{tab:method_comparison}
\setlength{\tabcolsep}{6pt}
\begin{tabular}{lcccc}
\toprule
\multirow{2}{*}{Method} & \multicolumn{2}{c}{$\alpha=0.05$} & \multicolumn{2}{c}{$\alpha=0.10$} \\
\cmidrule(lr){2-3} \cmidrule(lr){4-5}
 & FLR (\%) & Power (\%) & FLR (\%) & Power (\%) \\
\midrule
Confidence threshold & 2.80 & 57.23 & 3.94 & 66.60 \\
FLR search & 4.51 & 63.11 & 9.65 & 75.90 \\
Conformal Labeling & 5.00 & \textbf{72.27} & 9.97 & \textbf{79.91} \\
\bottomrule
\end{tabular}
\end{table}

\begin{table*}[t]
\centering
\caption{\textbf{Comparison with FLR-search baseline on open-ended generation tasks.}
We compare FLR Search with Conformal Labeling at target level $\alpha=0.05$.
All values are percentages.
\textbf{Bold} indicates the best power for each model--dataset pair.
}
\label{tab:flr_search_oe}
\footnotesize

\setlength{\tabcolsep}{4pt}
\renewcommand{\arraystretch}{1.12}

\begin{tabular}{ll *{2}{cc}}
\toprule
\multirow{2}{*}{Dataset}
& \multirow{2}{*}{Model}
& \multicolumn{2}{c}{FLR Search}
& \multicolumn{2}{c}{Conformal Labeling} \\
\cmidrule(lr){3-4} \cmidrule(lr){5-6}
& & FLR (\%) & Power (\%) & FLR (\%) & Power (\%) \\
\midrule

MATH-500
& Qwen3-4B-Instruct-2507
& 3.07 & 61.08 & 4.91 & \textbf{94.75} \\
& Qwen3-4B-Thinking-2507
& 4.01 & 82.44 & 4.73 & \textbf{95.27} \\
& DeepSeek-R1-Distill-Qwen-32B
& 4.31 & 28.59 & 4.83 & \textbf{47.82} \\

\midrule
MATH-L5
& Qwen3-4B-Instruct-2507
& 4.28 & 57.46 & 4.97 & \textbf{78.39} \\
& Qwen3-4B-Thinking-2507
& 3.00 & 70.59 & 4.12 & \textbf{97.88} \\
& DeepSeek-R1-Distill-Qwen-32B
& 5.95 & 9.19 & 4.72 & \textbf{51.43} \\

\midrule
HumanEval+
& Qwen3-4B-Instruct-2507
& 2.71 &45.51 & 4.78 & \textbf{75.46} \\
& Qwen3-4B-Thinking-2507
& 1.36 & 43.07 & 2.43 & \textbf{100.00} \\
& DeepSeek-R1-Distill-Qwen-32B
& 0.67 & 100.00 & 0.67 & \textbf{100.00} \\

\bottomrule
\end{tabular}
\end{table*}

\subsection{Performance of Conformal Labeling under different $n_0$}\label{subsec:n_0}

In Figure~\ref{fig:n0_variance}, we study the effect of the number of mislabeled calibration samples $n_0 = |\mathcal{D}_{\mathrm{cal}}^0|$.
For each run, we keep adding calibration samples until $n_0$ reaches a prescribed threshold in $\{10,50,100\}$, and then apply Conformal Labeling to the remaining test samples.

The results show that Conformal Labeling maintains valid FLR control across different values of $n_0$, while larger $n_0$ leads to more stable selection performance.
When $n_0=10$, the realized FLR remains close to or below the target level, but both FLR and power exhibit relatively large variance due to the limited number of mislabeled calibration samples.
Increasing $n_0$ from $10$ to $50$ substantially reduces the variance, as reflected by the narrower shaded regions, and also yields a modest improvement in power.
Further increasing $n_0$ to $100$ brings additional stability, although the improvement becomes more moderate.
Overall, these results suggest that Conformal Labeling can be applied with a small number of mislabeled calibration samples, while a moderate $n_0$ provides more reliable performance.

\begin{figure}
    \centering
    \includegraphics[width=1\linewidth]{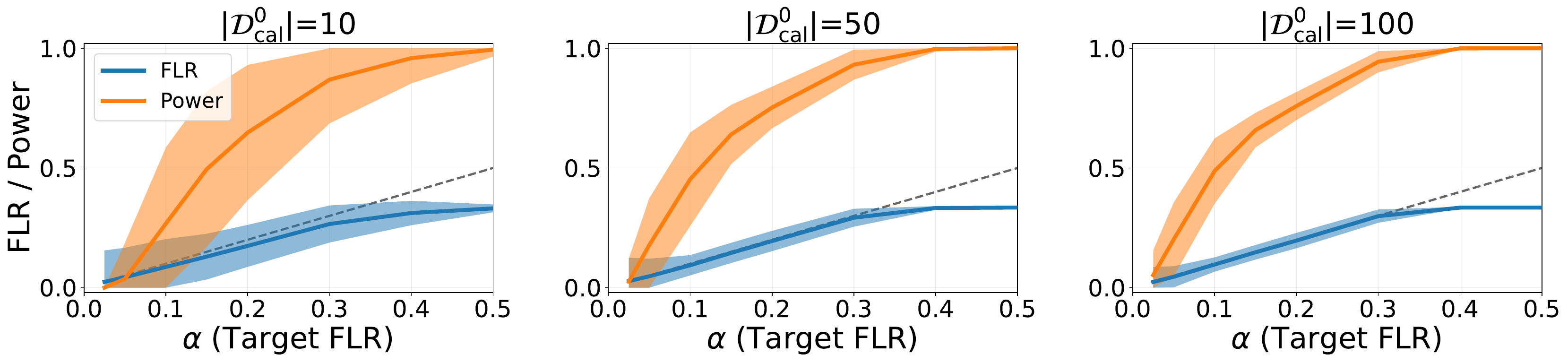}
    \caption{
    Performance of Conformal Labeling under different numbers of mislabeled calibration samples.
    We vary $n_0 = |\mathcal{D}_{\mathrm{cal}}^0|$ in $\{10,50,100\}$ and report the realized FLR and power across different target FLR levels $\alpha$.
    The dashed line denotes the nominal target level, and shaded regions represent standard deviations over repeated trials.
    }
    \label{fig:n0_variance}
\end{figure}

\subsection{Extension to group-wise FLR control}
\label{subsec:group-wise-FLR}

Conformal Labeling is designed to control the overall FLR. 
However, this does not ensure that the FLR is controlled within every subgroup: some groups may have group-wise FLR above the target level. 
Figure~\ref{fig:overall_cl_subject_FLR} illustrates this phenomenon on three representative MMLU subjects. 
While the overall FLR is controlled, the group-wise FLR varies across subjects, and the Machine Learning subject exceeds the target level.

Our method can be easily extended to control group-wise FLR for each group.
When group labels are available, we partition the calibration and test data according to group membership and apply Conformal Labeling within each group. 
Under the same i.i.d. assumption within each group, the FLR guarantee applies separately to each group. 
As shown in Figure~\ref{fig:groupwise_cl_subject_FLR}, the group-wise extension controls the FLR for all three subjects, demonstrating that Conformal Labeling can provide subgroup-level reliability guarantees when needed.


\begin{figure}[t]
    \centering
    \begin{subfigure}[t]{0.48\linewidth}
        \centering
        \includegraphics[width=\linewidth]{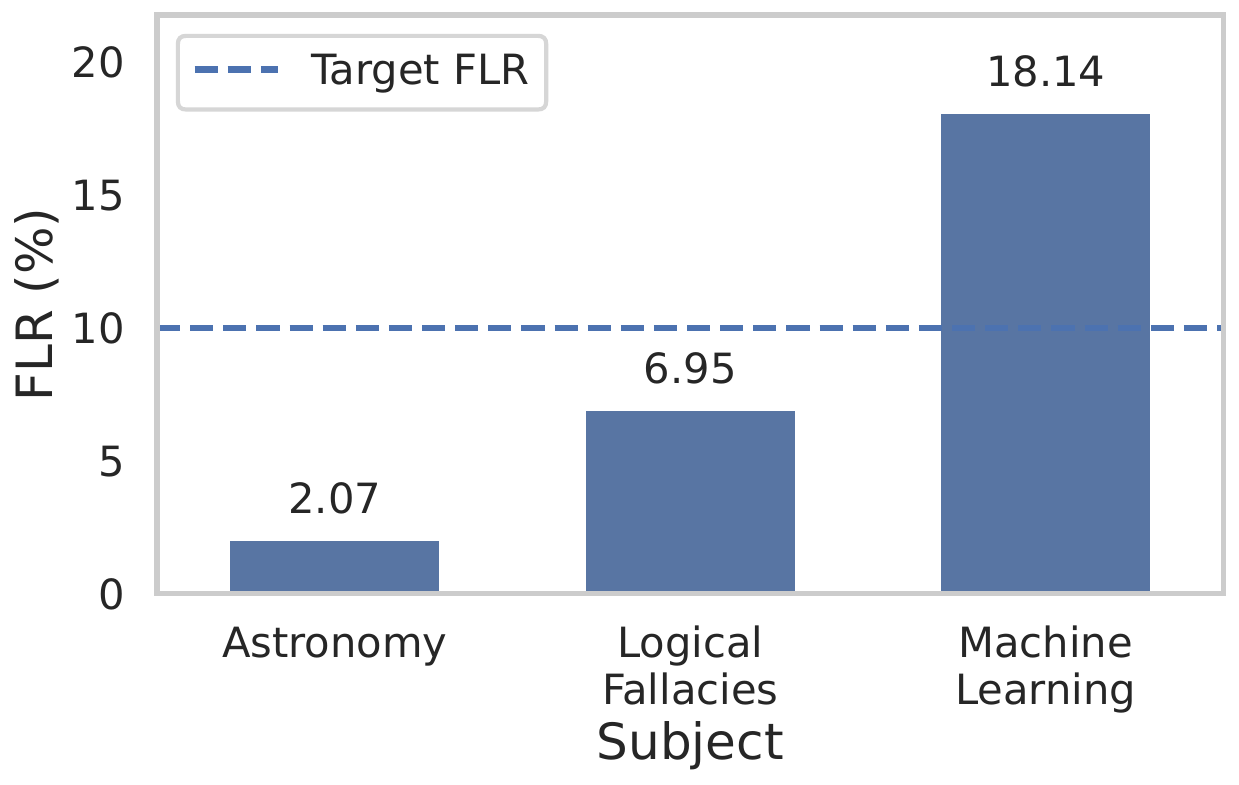}
        \caption{Conformal Labeling}
        \label{fig:overall_cl_subject_FLR}
    \end{subfigure}
    \hfill
    \begin{subfigure}[t]{0.48\linewidth}
        \centering
        \includegraphics[width=\linewidth]{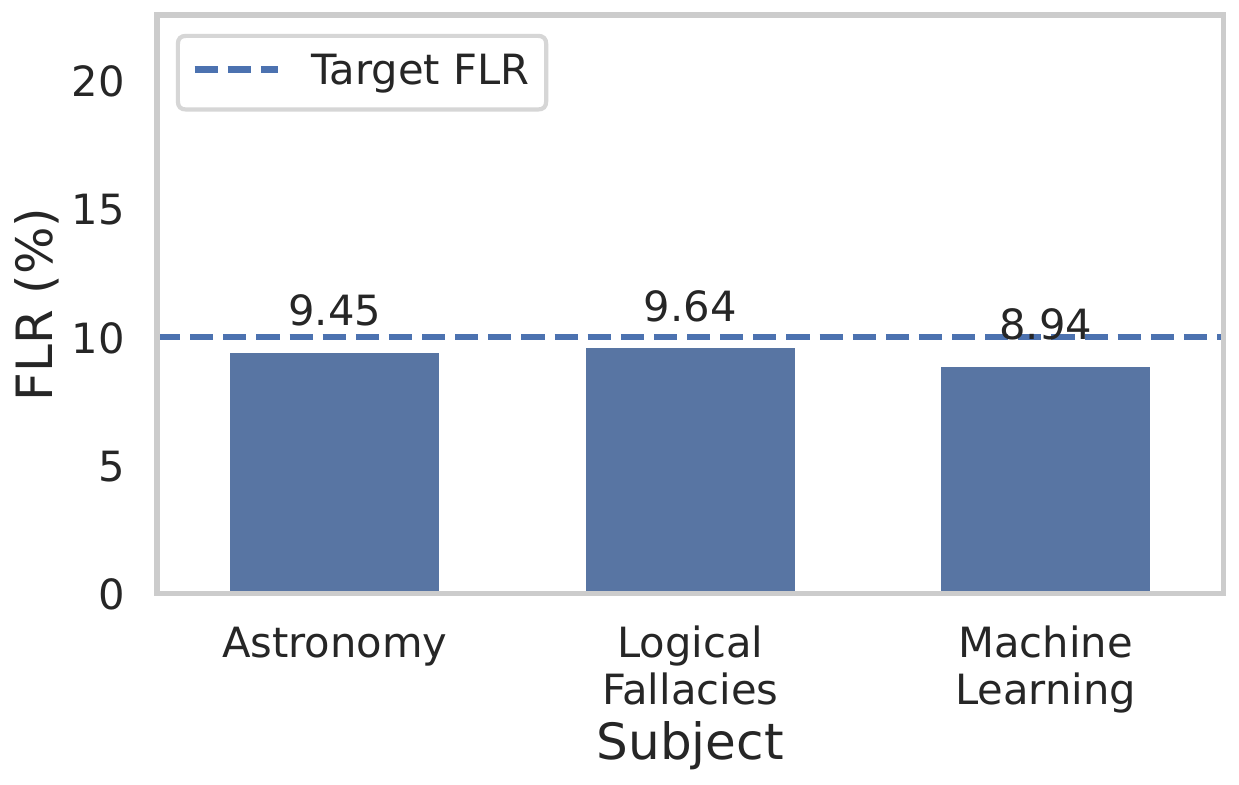}
        \caption{Group-wise Conformal Labeling}
        \label{fig:groupwise_cl_subject_FLR}
    \end{subfigure}
    \caption{
\textbf{Group-wise FLR on three representative MMLU subjects.}
The experiments are conducted with Qwen3-32B at $\alpha=0.1$.
Standard Conformal Labeling controls the overall FLR, while the realized FLR may vary across subjects.
Group-wise Conformal Labeling successfully controls the group-wise FLR below the target level for each group.
}
    \label{fig:groupwise_FLR}
\end{figure}

\subsection{Discussion of the Edge Case $n_0=0$}\label{subsec:edge_case}

Recall that $n_0$ denotes the number of incorrectly auto-labeled examples in the calibration set. When $n_0=0$, there are no incorrectly auto-labeled calibration examples, and the conformal $p$-value in Eq.~\eqref{eq:conformal_p} reduces to
\[
\hat p_j = U_j, \qquad U_j \sim \mathrm{Uniform}(0,1),
\]
for every test label. Consequently, regardless of whether the $j$-th test example belongs to the null hypothesis or the alternative hypothesis, its $p$-value follows the same uniform distribution:
\[
\hat p_j \mid H_j^0 \sim \mathrm{Uniform}(0,1),
\qquad
\hat p_j \mid H_j^1 \sim \mathrm{Uniform}(0,1).
\]
Thus, in this special case, the conformal $p$-values do not separate correct and incorrect auto-generated labels, and the uncertainty score no longer contributes to distinguishing null examples from alternative examples because there are no null calibration scores against which test scores can be compared. 
This edge case can occur in practical applications when the auto-labeling model is highly accurate, so that no errors are observed in the calibration set.

But Conformal Labeling can still achieve high power in this edge case when the calibration set has a moderate size.
To understand the selection behavior in this case, recall that Conformal Labeling selects labels according to the step-up rule
\[
j^* = \max \left\{ j : \hat{p}_{(j)} \le \frac{\alpha j (n+1)}{m(n_0+1)} \right\}.
\]
Although the $p$-values no longer separate null and alternative examples in this special case, the procedure can still select a large number of labels when 
$
\frac{\alpha(n+1)}{m(n_0+1)}
$
is sufficiently large.
In particular, when the calibration set has a moderate size, the factor $\frac{n+1}{n_0 + 1}$ becomes large when $n_0=0$, leading to high selection power.
Below, we provide an empirical study to illustrate this behavior.

\paragraph{Conformal Labeling remains effective when applied to highly accurate AI models.}
To empirically verify this behavior, we conduct an experiment on HumanEval+ using DeepSeek-R1-Distill-Qwen-32B. 
This setting corresponds to a near-perfect auto-labeling model, with an overall accuracy of $99.39\%$. 
We vary the calibration set size and evaluate the performance of Conformal Labeling conditional on the event $n_0=0$.
This conditional evaluation directly examines the edge case where no incorrectly auto-labeled examples are observed in the calibration set.
As shown in Table~\ref{tab:n0_zero_humaneval}, although the conformal $p$-values are uniformly distributed and therefore non-discriminative between null and alternative examples when $n_0=0$, Conformal Labeling can still achieve high power once the calibration set size is moderate. 
In particular, the power increases rapidly as the calibration size grows, from $1.73\%$ at $n=50$ to $100.00\%$ at $n=100$.
Overall, these results demonstrate that Conformal Labeling remains effective when applied to highly accurate AI models.

\begin{table}[t]
\centering
\caption{Performance of Conformal Labeling on HumanEval+ using DeepSeek-R1-Distill-Qwen-32B, conditional on the event $n_0=0$. The target FLR level is $\alpha=0.01$, and the overall auto-labeling accuracy is $99.39\%$.
All values are percentages.
Even though the conformal $p$-values are non-discriminative when no calibration errors are observed, our method still achieves high power with a moderate calibration set size.}
\label{tab:n0_zero_humaneval}
\begin{tabular}{lcc}
\toprule
Calibration size $n$ & FLR $\mid\, n_0=0$ & Power $\mid\, n_0=0$ \\
\midrule
50  & 0.46 & 1.73 \\
80  & 0.92 & 16.98 \\
90  & 1.25 & 41.78 \\
100 & 1.56 & 100.00 \\
\bottomrule
\end{tabular}
\end{table}


\subsection{End-to-End Experiment with Conformal Labeling}
\label{subsec:end_to_end}
To assess whether Conformal Labeling can facilitate training a better downstream model, we conduct an end-to-end experiment that integrates label selection with classifier retraining. 
Specifically, we evaluate whether the samples selected by Conformal Labeling can effectively augment the training set and improve prediction performance.
We adopt a synthetic dataset in this experiment to enable a controlled evaluation setting.
In particular, synthetic data allows us to generate a sufficiently large number of i.i.d.\ samples from a fixed distribution, which enables us to better analyze the end-to-end behavior of conformal labeling.

We generate a synthetic 10-class classification dataset with 27{,}000 samples using \texttt{scikit-learn}'s \texttt{make\_classification} function. The dataset is randomly partitioned into three disjoint subsets:
\begin{itemize}
    \item $D_1$: 2{,}000 samples used to train the initial labeling model;
    \item $D_2$: 20{,}000 samples used for AI labeling;
    \item $D_3$: 5{,}000 samples reserved for evaluation.
\end{itemize}

We first train a base classifier on $D_1$. From $D_2$, we randomly select 2{,}000 labeled samples as calibration data and apply Conformal Labeling with an FLR of $\alpha = 0.1$ to the remaining 18{,}000 unlabeled samples. 
This procedure yields an AI-labeled subset $D_{\text{label}}$.
We repeat this process using three different base classifiers: KNN, MLP, and XGBoost.
Table~\ref{tab:selection} reports the empirical false labeling proportion (FLP) and the number of selected samples for each classifier. In all cases, the observed FLP is well controlled below the target level, demonstrating the effectiveness of Conformal Labeling in selecting reliable pseudo-labeled data.
\begin{table}[t]
    \centering
    \caption{Selection results of conformal labeling with target FLR $\alpha = 0.1$.}
    \label{tab:selection}
    \begin{tabular}{lcc}
        \toprule
        Classifier & FLP (\%) & Selection Size \\
        \midrule
        KNN ($k=5$) & 9.72 & 6{,}142 \\
        MLP (100, 25) & 8.85 & 5{,}715 \\
        XGBoost & 9.56 & 4{,}006 \\
        \bottomrule
    \end{tabular}
\end{table}
After obtaining $D_{\text{label}}$, we retrain each classifier using the combined dataset $D_1 \cup D_{\text{label}}$. Table~\ref{tab:accuracy} compares the prediction accuracy on the held-out evaluation set $D_3$ between models trained only on $D_1$ and those trained on the augmented dataset.

The results show that incorporating the selected samples consistently improves downstream performance. For example, when KNN is used as the base classifier, training on $D_1 \cup D_{\text{label}}$ increases test accuracy from 64.92\% to 69.60\% (+4.68\%). Similar gains are observed for MLP (+2.60\%) and XGBoost (+0.26\%). 
These findings indicate that our method enables effective utilization of unlabeled data to train more accurate models in some scenarios.

\begin{table}[t]
    \centering
    \caption{Prediction accuracy on $D_3$ with and without conformal-labeled samples.}
    \label{tab:accuracy}
    \begin{tabular}{lccc}
        \toprule
        Classifier & Training Set & Accuracy (\%) & Improvement (\%) \\
        \midrule
        KNN & $D_1$ & 64.92 & -- \\
            & $D_1 \cup D_{\text{label}}$ & \textbf{69.60} & \textbf{+4.68} \\
        \midrule
        MLP & $D_1$ & 65.68 & -- \\
            & $D_1 \cup D_{\text{label}}$ & \textbf{68.28} & \textbf{+2.60} \\
        \midrule
        XGBoost & $D_1$ & 59.80 & -- \\
            & $D_1 \cup D_{\text{label}}$ & \textbf{60.06} & \textbf{+0.26} \\
        \bottomrule
    \end{tabular}
\end{table}

\subsection{Discussion of Conformal Alignment}
\label{subsec:conformal_alignment}

Conformal selection (CS)~\citep{jin2023selection} is a screening framework that combines conformal $p$-values with the BH procedure to select candidates whose unobserved outcomes exceed user-specified thresholds.
Conformal alignment (CA)~\citep{gui2024conformal} adapts this framework to alignment tasks, where it trains an alignment score predictor as the nonconformity score function and selects foundation model outputs whose unobserved alignment scores exceed a given threshold.
Our work is motivated by the framework of conformal selection and tailors it to selective auto-labeling.
However, CS and CA cannot be directly applied to the labeling problem.
Specifically, our method differs from CS and CA in the following three aspects.

\paragraph{Formulation and assumptions.}
CS and CA are designed to select instances whose unobserved outcomes, which are continuous variables (e.g., the alignment score in CA), exceed user-specified values.
In contrast, our method aims to select instances whose auto-generated labels are correct, where the correctness is a discrete variable.
Moreover, CS assumes that the nonconformity score $V(x,y)$ is monotone in the outcome $y$, and CA further requires a separate training set to fit the alignment score predictor.
In contrast, our method uses an uncertainty score $\mathcal{S}(x)$ that does not involve the outcome and does not require a training set.

\paragraph{Construction of $p$-values.}
CS and CA construct conformal $p$-values using the full calibration set, and establish a marginal validity guarantee:
\[
\mathbb{P}\left(\hat{p}_j \le t,\ H_j^0 \text{ is true}\right) \le t, \quad \forall t\in[0,1].
\]
In contrast, we construct conformal $p$-values using only the mislabeled calibration samples in $\mathcal{D}_{\mathrm{cal}}^0$, which yields a stronger guarantee conditional on each null hypothesis:
\[
\mathbb{P}\left(\hat{p}_j \le t \mid H_j^0 \text{ is true}\right) \le t, \quad \forall t\in[0,1].
\]
Thus, the two constructions provide distinct validity guarantees tailored to different testing problems.

\paragraph{Selection procedure.}
CS/CA only supports using standard BH without correction. Applying BH with a correction factor (e.g., the Storey correction~\citep{storey2002direct}) to their $p$-values is not theoretically justified.
In contrast, our method can apply a novel correction factor to the BH procedure due to the strong null validity of the proposed $p$-value.
As shown in Table~\ref{tab:compare_selection_results}, this design substantially improves power over the standard BH procedure.

\subsection{Discussion of Colander}
\label{subsec:colander}

Colander~\citep{vishwakarma2024pearls} is a threshold-based auto-labeling (TBAL) method that also aims to maximize the coverage of auto-labeled data subject to an auto-labeling error constraint.
In each round of the TBAL protocol, Colander retrains the classifier, learns a confidence function jointly with the selection threshold on held-out validation data, and selects the threshold whose estimated error plus a standard-deviation-based safety margin falls below the target level.
While this margin helps control the error empirically, it provides no guarantee on the error of the auto-labeled subset.
In contrast, Conformal Labeling fixes the labeling model and selects the subset in a single step, which allows it to provide finite-sample FLR control.
Below, we compare Conformal Labeling and Colander in two scenarios: (1) the selective auto-labeling setting considered in our work, and (2) the TBAL protocol adopted by Colander.

\paragraph{Comparison in the selective auto-labeling setting.}
We use MATH-500 with Qwen3-4B-Instruct-2507, the same calibration and test splits for both methods, and a target FLR of $\alpha=0.05$.
As shown in Table~\ref{tab:colander_selective}, both methods control the FLR below the target level, while Conformal Labeling achieves substantially higher power than Colander ($94.82\%$ vs.\ $38.47\%$).

\paragraph{Comparison under the iterative TBAL protocol.}
We further follow the iterative TBAL protocol of Colander, using MNIST with a LeNet classifier, $200$ initial labeled samples, $20$ newly queried labels per round, and $\epsilon_{\alpha}=\alpha=0.05$.
We replace only the threshold-selection procedure of Colander with Conformal Labeling and keep all other components unchanged.
Results in Table~\ref{tab:colander_tbal} show that Colander controls the auto-labeling error, whereas Conformal Labeling slightly exceeds the target level. 
This discrepancy may arise from two reasons: first, the i.i.d. assumption may be violated during the iterative process; second, controlling the FLR at level $\alpha$ for two round-wise selected sets does not theoretically imply that their union has expected FLR at most  $\alpha$.

\paragraph{Discussion.}
These results clarify that Colander and Conformal Labeling target different scopes and are complementary rather than competing.
Colander is designed for the iterative TBAL protocol, where its learned confidence function and safety margin enable empirical error control under repeated model updates.
Conformal Labeling is designed for single-step selection with a fixed labeling model, where it provides rigorous finite-sample FLR control and substantially higher power.
Moreover, Conformal Labeling requires neither retraining nor access to model internals for learning a confidence function, making it directly applicable to modern LLM-based labeling pipelines.

\begin{table}[t]
\centering
\caption{\textbf{Comparison between Colander and Conformal Labeling in the selective auto-labeling setting.}
Results are obtained on MATH-500 with Qwen3-4B-Instruct-2507 at target level $\alpha=0.05$.
All values are percentages.
\textbf{Bold} indicates the best power.}
\label{tab:colander_selective}
\setlength{\tabcolsep}{8pt}
\begin{tabular}{lcc}
\toprule
Method & FLR (\%) & Power (\%) \\
\midrule
Colander & 2.18 & 38.47 \\
Conformal Labeling & 4.89 & \textbf{94.82} \\
\bottomrule
\end{tabular}
\end{table}

\begin{table}[t]
\centering
\caption{\textbf{Comparison between Colander and Conformal Labeling under the iterative TBAL protocol.}
Results are obtained on MNIST with a LeNet classifier, where only the threshold-selection procedure of Colander is replaced by Conformal Labeling.
All values are percentages.}
\label{tab:colander_tbal}
\setlength{\tabcolsep}{8pt}
\begin{tabular}{lcc}
\toprule
Method & Auto-labeling Error (\%) & Coverage (\%) \\
\midrule
Colander & 3.45 & 83.74 \\
Conformal Labeling & 5.61 & 94.33 \\
\bottomrule
\end{tabular}
\end{table}

\subsection{Discussion of marginal FLR and FLR control}
\label{subsec:mFLR_vs_FLR}

Let
\[
V_m := \sum_{j=1}^m \mathbf{1}\{Y_{n+j} \neq \hat{Y}_{n+j},\, j \in \mathcal{R}\},
\qquad
R_m := \sum_{j=1}^m \mathbf{1}\{j \in \mathcal{R}\},
\]
where \(V_m\) denotes the number of incorrect auto-generated labels in the selected subset, and \(R_m\) denotes the number of selected labels.
In the labeling task, the marginal FLR (mFLR) is defined as
\[
\mathrm{mFLR} = \frac{\mathbb{E}[V_m]}{\mathbb{E}[R_m]},
\]
whereas FLR is defined as
\[
\mathrm{FLR}
=
\mathbb{E}\!\left[
\frac{V_m}{\max\{R_m,1\}}
\right].
\]

The FLR controls the expected realized false labeling proportion, whereas mFLR
controls only the ratio of the two expectations and therefore does not directly
control the joint behavior of \(V_m\) and \(R_m\). 
This distinction is well studied in the multiple-testing literature, where it corresponds to the distinction between the false discovery rate (FDR) and the marginal false discovery rate (mFDR).
For example, \citet{benjamini1995controlling} proposed the concept of FDR in favor of mFDR, because the latter is impossible to control in the global null case \citep{gui2024conformal}.
\citet{storey2003positive} also pointed out that mFDR cannot be used for controlling the joint behavior of the numerator (number
of false selections) and the denominator (number of selections).

\subsection{Prompt Sensitivity Analysis}
\label{subsec:prompt_sensitivity}

To evaluate the sensitivity of our method to prompt design, we conduct experiments using Qwen3-32B on the MMLU dataset with three distinct prompts. The confidence measure is computed using a logits-based approach. The three prompts share identical question and answer content but differ in their phrasing and output format constraint. We describe each prompt in detail below.

\paragraph{Prompt 1.}
The model is instructed to answer a multiple-choice question by outputting only the correct option without any additional text or explanation:
\begin{quote}
\small
\texttt{The following are multi choice questions. Give ONLY the correct option, no other words or explanation:\\
Question: \{question\}\\
A: \{choice1\}\\
B: \{choice2\}\\
C: \{choice3\}\\
D: \{choice4\}\\
Answer:}
\end{quote}

\paragraph{Prompt 2.}
The model is explicitly asked to respond with only the letter corresponding to the correct choice, again without providing explanations:
\begin{quote}
\small
\texttt{You will be given multiple-choice questions. Respond with ONLY the letter of the correct choice. No explanations.\\
\\
Question: \{question\}\\
\\
A: \{choice1\}\\
B: \{choice2\}\\
C: \{choice3\}\\
D: \{choice4\}\\
\\
Answer:}
\end{quote}

\paragraph{Prompt 3.}
The model is prompted to output only the correct option, with a different structural layout that groups the answer choices under an ``Options'' header:
\begin{quote}
\small
\texttt{Answer the following multiple-choice question. Output ONLY the correct option (A, B, C, etc.). No other text.\\
\\
Question: \{question\}\\
\\
Options:\\
A: \{choice1\}\\
B: \{choice2\}\\
C: \{choice3\}\\
D: \{choice4\}\\
\\
Correct option:}
\end{quote}

Table~\ref{tab:prompt_sensitivity} reports the performance of our method under the three prompt designs in terms of realized FLR, power, and the testing error of the underlying language model.
The results show that the mean FLR remains consistently close to the nominal level across all prompt designs, indicating that the FLR guarantee of our method is largely insensitive to prompt formulation.
Besides, we notice that prompt 3 leads to lower power than the other two prompts.
This is because the LLM achieves a higher error with this prompt. 
Overall, our method can be employed with LLMs under different prompts.

\begin{table}[t]
\centering
\caption{Performance of our method under different prompt designs on the MMLU dataset.}
\label{tab:prompt_sensitivity}
\begin{tabular}{c|c|c|c}
\hline
Prompt & Mean FLR (\%) & Mean Power (\%) & Error (\%) \\
\hline
1 & 9.97 & 82.99 & 21.46 \\
2 & 9.94 & 82.58 & 21.48 \\
3 & 10.00 & 79.18 & 22.10 \\
\hline
\end{tabular}
\end{table}

\section*{NeurIPS Paper Checklist}

\begin{enumerate}

\item {\bf Claims}
    \item[] Question: Do the main claims made in the abstract and introduction accurately reflect the paper's contributions and scope?
    \item[] Answer: \answerYes{} 
    \item[] Justification: Our experimental results consistently support the claims presented in the abstract and introduction.

\item {\bf Limitations}
    \item[] Question: Does the paper discuss the limitations of the work performed by the authors?
    \item[] Answer: \answerYes{} 
    \item[] Justification: We discuss the limitation of this work in section~\ref{section:limitation}.

\item {\bf Theory assumptions and proofs}
    \item[] Question: For each theoretical result, does the paper provide the full set of assumptions and a complete (and correct) proof?
    \item[] Answer: \answerYes{} 
    \item[] Justification: We provide the proof of our theoretical result in the Appendix~\ref{section:proof}.

    \item {\bf Experimental result reproducibility}
    \item[] Question: Does the paper fully disclose all the information needed to reproduce the main experimental results of the paper to the extent that it affects the main claims and/or conclusions of the paper (regardless of whether the code and data are provided or not)?
    \item[] Answer: \answerYes{} 
    \item[] Justification: We provide the implementation details to reproduce our results.

\item {\bf Open access to data and code}
    \item[] Question: Does the paper provide open access to the data and code, with sufficient instructions to faithfully reproduce the main experimental results, as described in supplemental material?
    \item[] Answer: \answerYes{} 
    \item[] Justification: We provide code and data for reproducing our main experiments \href{https://github.com/HuipengHuang/conformal_labeling}{https://github.com/ml-stat-Sustech/conformal\_labeling}.
    \item[] Guidelines:
    \begin{itemize}
        \item The answer \answerNA{} means that paper does not include experiments requiring code.
        \item Please see the NeurIPS code and data submission guidelines (\url{https://github.com/ml-stat-Sustech/conformal_labeling}) for more details.
        \item While we encourage the release of code and data, we understand that this might not be possible, so \answerNo{} is an acceptable answer. Papers cannot be rejected simply for not including code, unless this is central to the contribution (e.g., for a new open-source benchmark).
        \item The instructions should contain the exact command and environment needed to run to reproduce the results. See the NeurIPS code and data submission guidelines (\url{https://github.com/ml-stat-Sustech/conformal_labeling}) for more details.
        \item The authors should provide instructions on data access and preparation, including how to access the raw data, preprocessed data, intermediate data, and generated data, etc.
        \item The authors should provide scripts to reproduce all experimental results for the new proposed method and baselines. If only a subset of experiments are reproducible, they should state which ones are omitted from the script and why.
        \item At submission time, to preserve anonymity, the authors should release anonymized versions (if applicable).
        \item Providing as much information as possible in supplemental material (appended to the paper) is recommended, but including URLs to data and code is permitted.
    \end{itemize}

\item {\bf Experimental setting/details}
    \item[] Question: Does the paper specify all the training and test details (e.g., data splits, hyperparameters, how they were chosen, type of optimizer) necessary to understand the results?
    \item[] Answer: \answerYes{} 
    \item[] Justification: We provide the data split information for open-ended generation tasks in Table~\ref{tab:datasets_split} and other information in
implementation details.

\item {\bf Experiment statistical significance}
    \item[] Question: Does the paper report error bars suitably and correctly defined or other appropriate information about the statistical significance of the experiments?
    \item[] Answer: \answerYes{} 
    \item[] Justification: Our results are provided in mean format.

\item {\bf Experiments compute resources}
    \item[] Question: For each experiment, does the paper provide sufficient information on the computer resources (type of compute workers, memory, time of execution) needed to reproduce the experiments?
    \item[] Answer: \answerYes{} 
    \item[] Justification:  The compute resources are provided in Appendix~\ref{section:moreimplementation}.
    
\item {\bf Code of ethics}
    \item[] Question: Does the research conducted in the paper conform, in every respect, with the NeurIPS Code of Ethics \url{https://neurips.cc/public/EthicsGuidelines}?
    \item[] Answer: \answerYes{} 
    \item[] Justification: Yes, we affirm that our research adheres to the NeurIPS Code of Ethics in all
aspects, including considerations related to data usage, transparency, and potential societal
impact.

\item {\bf Broader impacts}
    \item[] Question: Does the paper discuss both potential positive societal impacts and negative societal impacts of the work performed?
    \item[] Answer: \answerYes{} 
    \item[] Justification: We discuss about the societal impacts of selective auto-labeling in the introduction. What's more, we also discuss why we need reliability guarantees for selective auto-labeling.
    
\item {\bf Safeguards}
    \item[] Question: Does the paper describe safeguards that have been put in place for responsible release of data or models that have a high risk for misuse (e.g., pre-trained language models, image generators, or scraped datasets)?
    \item[] Answer: \answerNA{} 
    \item[] Justification: Our work does not involve the release of model and dataset. Our models and
datasets are all from Torch or Hugging Face.

\item {\bf Licenses for existing assets}
    \item[] Question: Are the creators or original owners of assets (e.g., code, data, models), used in the paper, properly credited and are the license and terms of use explicitly mentioned and properly respected?
    \item[] Answer: \answerYes{} 
    \item[] Justification: The dataset and model we use are both open-sourced. Our
code is based on PyTorch and vLLM.

\item {\bf New assets}
    \item[] Question: Are new assets introduced in the paper well documented and is the documentation provided alongside the assets?
    \item[] Answer: \answerYes{} 
    \item[] Justification: We provide code and data for reproducing our main experiments \href{https://github.com/HuipengHuang/conformal_labeling}{https://github.com/ml-stat-Sustech/conformal\_labeling}.

\item {\bf Crowdsourcing and research with human subjects}
    \item[] Question: For crowdsourcing experiments and research with human subjects, does the paper include the full text of instructions given to participants and screenshots, if applicable, as well as details about compensation (if any)? 
    \item[] Answer: \answerNA{} 
    \item[] Justification: Our work does not involve crowdsourcing nor research with human subjects.

\item {\bf Institutional review board (IRB) approvals or equivalent for research with human subjects}
    \item[] Question: Does the paper describe potential risks incurred by study participants, whether such risks were disclosed to the subjects, and whether Institutional Review Board (IRB) approvals (or an equivalent approval/review based on the requirements of your country or institution) were obtained?
    \item[] Answer: \answerNA{} 
    \item[] Justification: Our work does not involve crowdsourcing nor research with human subjects.

\item {\bf Declaration of LLM usage}
    \item[] Question: Does the paper describe the usage of LLMs if it is an important, original, or non-standard component of the core methods in this research? Note that if the LLM is used only for writing, editing, or formatting purposes and does \emph{not} impact the core methodology, scientific rigor, or originality of the research, declaration is not required.
    \item[] Answer: \answerNA{} 
    \item[] Justification: The core method development in this research does not involve LLMs as any
important, original, or non-standard components.

\end{enumerate}

\end{document}